\documentclass[10pt,journal,compsoc]{IEEEtran}
\ifCLASSOPTIONcompsoc
  \usepackage[nocompress]{cite}
\else
  \usepackage{cite}
\fi
\ifCLASSINFOpdf
\else
\fi
\usepackage{microtype}
\usepackage{graphicx}
\usepackage{bm}
\usepackage{algorithm}
\usepackage{algorithmic}
\usepackage{multirow}
\usepackage{amsthm}
\usepackage{amsmath}
\usepackage{subfigure}
\usepackage{amssymb}
\usepackage{bbding}
\usepackage{pifont}
\usepackage{wasysym}
\usepackage{booktabs} 
\usepackage{makecell}

\usepackage{balance}
\usepackage{float}
\usepackage{subeqnarray}
\usepackage{cases}
\usepackage{tabularx}
\usepackage{color}
\usepackage{colortbl,url}
\usepackage{threeparttable}
\newcommand{\cmark}{\ding{51}}%
\newcommand{\xmark}{\ding{55}}%

\newcommand{\tabincell}[2]{\begin{tabular}{@{}#1@{}}#2\end{tabular}}

\newtheorem{thm}{Theorem}

\def\x{\mathbf{x}}
\def\y{\mathbf{y}}

\def\Y{\mathcal{Y}}

\def\X{\mathcal{X}}
\def\S{\mathcal{S}}

\def\A{\mathbf{A}}
\def\B{\mathbf{B}}

\begin{document}
 
\title{Investigating Bi-Level Optimization for Learning and Vision from a Unified Perspective: \\A Survey and Beyond}
 
\author{Risheng~Liu,~\IEEEmembership{Member,~IEEE,}
        Jiaxin~Gao, 
        Jin~Zhang,
        Deyu~Meng,~\IEEEmembership{Member,~IEEE,}
        and~Zhouchen~Lin,~\IEEEmembership{Fellow,~IEEE}
\IEEEcompsocitemizethanks{\IEEEcompsocthanksitem  R. Liu and J. Gao are with the DUT-RU International School of Information Science \& Engineering, Dalian University of Technology, and also with the Key Laboratory for Ubiquitous Network and Service Software of Liaoning Province, Dalian 116024, China.
E-mail: rsliu@dlut.edu.cn, jiaxinn.gao@outlook.com. R. Liu is the corresponding author. \protect 
\IEEEcompsocthanksitem J. Zhang is with the Department of Mathematics, Southern University of Science and Technology, and National Center for Applied Mathematics Shenzhen, China, E-mail: zhangj9@sustech.edu.cn.\protect
\IEEEcompsocthanksitem D. Meng is with School of Mathematics and Statistics and Ministry of Education Key Lab of Intelligent Networks and Network Security, Xi’an Jiaotong University, Xi'an, Shaanxi, China. E-mail: dymeng@mail.xjtu.edu.cn.\protect
\IEEEcompsocthanksitem Z. Lin is with the Key Laboratory of Machine Perception (Ministry of Education), School of Electronics Engineering and Computer Science, Peking University, Beijing 100871, China, and also with the Cooperative Medianet Innovation Center, Shanghai Jiao Tong University, Shanghai 200240, China. E-mail: zlin@pku.edu.cn.\protect
}
\thanks{Manuscript received April 19, 2005; revised August 26, 2015.}}

%
%

\markboth{Journal of \LaTeX\ Class Files,~Vol.~14, No.~8, August~2015}%
{Shell \MakeLowercase{\textit{et al.}}: Bare Demo of IEEEtran.cls for Computer Society Journals}

\IEEEtitleabstractindextext{%
\begin{abstract}
Bi-Level Optimization (BLO) is originated from the area of economic game theory and then introduced into the optimization community. BLO is able to handle problems with a hierarchical structure, involving two levels of optimization tasks, where one task is nested inside the other. In machine learning and computer vision fields, despite the different
motivations and mechanisms, a lot of complex problems, such as hyper-parameter optimization, multi-task and meta learning, neural architecture search, adversarial learning and deep reinforcement learning, actually all contain a series of closely related subproblms. In this paper, we first uniformly express these complex learning and vision problems from the perspective of BLO. Then we construct a best-response-based single-level reformulation and 
establish a unified algorithmic framework to understand and formulate mainstream gradient-based BLO methodologies, covering aspects ranging from fundamental automatic differentiation schemes to various accelerations, simplifications, extensions and their convergence and complexity properties.
Last but not least, we discuss the potentials of our unified BLO framework for designing new algorithms and point out some promising directions for future research.
\end{abstract}

}

\maketitle

\IEEEdisplaynontitleabstractindextext

\IEEEpeerreviewmaketitle

\IEEEraisesectionheading{\section{Introduction}\label{sec:introduction}}

\IEEEPARstart{B}{i-Level} Optimization (BLO) is the hierarchical mathematical program where the feasible region of one optimization task is restricted by the solution set mapping of another optimization task (i.e., the second task is embedded within the first one)~\cite{dempe2020bilevel}. The outer optimization task is commonly referred to as the Upper-Level (UL) problem, and the inner optimization task is commonly referred to as the Lower-Level (LL) problem. BLOs involve two kinds of variables, referred to as the UL and LL variables, accordingly.

The origin of BLOs can be traced to the domain of game theory and is known as Stackelberg competition~\cite{von1952theory}. 
Subsequently, it has been investigated in view of many important applications in various fields of science and engineering, particularly in economics, management, chemistry, optimal control, and resource allocation problems~\cite{fortuny1981representation,dempe2015bilevel,sinha2017review,wogrin2020applications}. Especially,
in recent years, a great amount of modern applications in the fields of machine learning and computer vision (e.g., hyper-parameter optimization~\cite{domke2012generic,foo2008efficient,mackay2019self,okuno2018hyperparameter}, multi-task and meta  learning~\cite{franceschi2017bridge,franceschi2018bilevel,shaban2019truncated}, neural architecture search~\cite{liu2018darts,hu2020tf,lian2019towards}, generative adversarial learning~\cite{li2019learning,jiang2018learning,tian2020alphagan}, deep reinforcement learning~\cite{zhang2020bi,pfau2016connecting,yang2019provably} and image processing and analysis~\cite{ochs2015bilevel,chen2020flexible,ijcai2020-101}, just name a few) have arisen that fit the BLO framework.

In general, most of the earlier BLOs are highly complicated and computationally challenging to solve due to their nonconvexity and non-differentiability~\cite{kunapuli2008classification,dempe2016solution}. Despite their apparent simplicity, BLOs are nonconvex problems with an implicitly determined feasible region even if the UL and LL subproblems are convex~\cite{hansen1992new,zemkoho2016solving}. 
Indeed, it has been proved that even strictly checking the local optimality of the simplest BLO model (e.g., linear BLO) is still a NP-hard problem~\cite{vicente1994descent,calvete2020algorithms}. In addition, the existence of multiple optima for the LL subproblem can result in an inadequate formulation of BLOs, which could aggravate the difficulty of theoretical analysis~\cite{liu2020generic}. Despite the challenges, a lot of research topics consisting of methods and applications of BLOs have followed in this field, see~\cite{dempe2016solution,sharma2020optimistic}. Early studies focused on numerical methods, including extreme-point methods~\cite{calvete2012bilevel}, branch-and-bound methods~\cite{fortuny1981representation,lu2007extended}, descent methods~\cite{savard1994steepest,neto2011perturbed}, penalty function methods~\cite{anandalingam1990solution,wan2014estimation}, trust-region methods~\cite{el2018penalty,dempe2001bundle}, and so on. The most often used procedure is to replace the LL subproblem with its Karush–Kuhn–Tucker (KKT) conditions, and if assumptions are made (such as smoothness, convexity, among others) the BLOs can be transformed into single-level optimization problems~\cite{allende2013solving,sinha2019using,sinha2018bilevel}. 
However, due to the high complexity of bi-level models, solving BLOs for large-scale and high-dimensional practical applications in learning and vision fields is still challenging~\cite{yimer2020proximal}.

The classical idea (e.g., the first-order approach in economics literature) to reformulate BLO is to replace the LL subproblem Eq.~\eqref{eq:blo-ll} by its KKT conditions and minimize over the original variables $\x$ and $\y$ as well as the multipliers. The resulting problem is a so-called Mathematical Program with Equilibrium Constraints (MPEC) \cite{raghunathan2003mathematical,zemkoho2020theoretical}.  Unfortunately, MPECs are still a challenging class of problems because of the presence of the complementarity
constraint~\cite{aussel2018generalized}. Solution methods for MPECs can be categorized into two types of approaches. The first one, namely, the nonlinear programming approach rewrites the complementarity constraint
into nonlinear inequalities, and then
allows to leverage powerful numerical nonlinear programming solvers. The other one, namely, the combinatorial approach tackles the combinatorial nature of the disjunctive constraint.
Despite the difficulties, MPEC has been studied intensively in the last three decades~\cite{kunapuli2008bilevel}. Recently,
some progress on the MPEC approach in dealing with BLOs have been witnessed by the community of mathematical programming, in the context of selecting optimal
hyper-parameters in regression and classification problems. There are two issues caused by the multipliers in the MPEC approach. \textcolor[rgb]{0.00,0.00,0.00} {First, in theory, if there exist more than one multipliers for the LL subproblem, MPEC will not be equivalent to the original BLO (in the local optimality scenario)~\cite{dempe2012bilevel}.} Second, the introduced auxiliary multiplier variables can limit the numerical efficiency when solving the BLO problem.

In recent years, a variety of machine learning and computer vision tasks, including but not limited to, hyper-parameter optimization~\cite{franceschi2017forward,okuno2020bilevel,shaban2019truncated,liu2020boml}), multi-task and meta  learning~\cite{ji2020multi,li2017meta,chen2017learning,antoniou2018train}, neural architecture search~\cite{liu2018darts,lian2019towards,dong2020automatic,xu2019pc}, adversarial learning~\cite{jiang2018learning,li2019learning,liu2020gl,pfau2016connecting}, and deep reinforcement learning~\cite{zhang2020bi,wang2020global,pfau2016connecting,hong2020two}, have been investigated in application scenarios. 
Despite the different motivations and mechanisms, all these problems contain a series of closely related subproblems and have a natural hierarchical optimization structure. However, although received increasing attentions in both academic and industrial communities, there still lack a unified perspective to understand and formulate these different categories of hierarchical learning and vision problems.  

We notice that most previous surveys on BLOs (e.g., ~\cite{dempe2020bilevel,anandalingam1992hierarchical,wen1991linear,ukkusuri2013bi,lachhwani2018bi,chinchuluun2009multilevel,gould2016differentiating}) are purely from the viewpoint of mathematical programming and mainly focus on the formulations, properties, optimality conditions and these classical solution algorithms, such as evolutionary methods~\cite{sinha2017review}. In contrast, the aim of this paper is to utilize BLO to express a variety of complex learning and vision problems, which explicitly or implicitly contain closely related subproblems. Furthermore, we present a unified perspective to comprehensively survey different categories of gradient-based BLO methodologies in specific learning and vision applications. In particular, we first provide a literature review on various complex learning and vision problems, including hyper-parameter optimization, multi-task and meta learning, neural architecture search, adversarial learning, deep reinforcement learning and so on. We demonstrate that all these tasks can be modeled as a general BLO formulation. Following this perspective, we then establish a best-response-based single-level reformulation to express these existing BLO models. By further introducing a unified algorithmic framework on the single-level reformulation, we can uniformly understand and formulate these existing gradient-based BLOs and analyze their accelerations, simplifications, extensions, and convergence and complexity proprieties. Finally, we demonstrate the potentials of our framework for designing new algorithms and point out some promising research directions for BLO in learning and vision fields.

Compared with existing surveys on BLOs, our major contributions can be summarized as follows:
\begin{enumerate}
	\item To the best of our knowledge, this is the first survey paper to focus on uniformly understanding and (re)formulating different categories of complex machine learning and computer vision tasks and their solution methods (especially in the context of deep learning) from the perspective of BLO.
	\item By introducing a best-response-based single-level reformulation and constructing a best-response-based algorithmic framework, we obtain a general and flexible platform that can successfully unify different existing gradient-based BLO methodologies and uniformly analyze these accelerations, simplifications, and extensions in literature.
	\item The convergence behaviors of gradient-based BLOs are comprehensively analyzed. Especially, we establish a general convergence analysis template to investigate the iteration behaviors of a series of gradient-based BLOs from a unified perspective. The time and space complexity of various mainstream schemes is also systematically analyzed. 
	\item Our gradient-based BLO platform not only comprehensively covers mainstream gradient-based BLO methods, but also has potentials for designing new BLO algorithms to deal with more challenging tasks. We also point out some promising directions for future research.
\end{enumerate}

We summarize our mathematical notations in Table~\ref{tab:notation-2}. The remainder of this paper is organized as follows. We first introduce some necessary fundamentals of BLOs in Section~\ref{sec:fundamental}. Then, Section~\ref{sec:alltasks} provides a comprehensive survey of various learning and vision applications that all can be modeled as BLOs. In Section~\ref{sec:framework}, we establish an algorithmic framework in a unified manner for existing gradient-based BLO schemes. Within this framework, we further understand and formulate two different categories of BLOs (i.e., explicit and implicit gradients for best-response) in Section~\ref{sec:egbr} and Section~\ref{sec:igbr}, respectively. We also discuss the so-called lower-level singleton issue of BLOs in Section~\ref{sec:beyond}. The convergence and complexity properties of these gradient-based BLOs are discussed in Section~\ref{sec:convergence}. Section~\ref{sec:pessimistic} puts forward potentials of our framework for designing new algorithms to deal with more challenging pessimistic BLOs. Finally, Section~\ref{sec:conclusions} points out some promising directions for future research.


\begin{table*}[ht]
	\centering
	\small
	\caption{Summary of mathematical notations.}\label{tab:notation-2}
	
	\begin{tabular}{|l l|l l|}
		\hline
		Notation & Description & Notation &  Description  \\
		\hline	\hline
		$\mathcal{D}_{\mathtt{tr}}$/$\mathcal{D}_{\mathtt{val}}$ & Training/Validation data & $o^{ij}/\mathbf{x}_{o}^{ij}$ & Operations/Operation weights\\
		$\pi/r$ & Policy/Reward &  $s/a$ &State/Action\\
		$Q^{\pi}$  & Q-function under $\pi$ & $Q^{\pi}(s,a)$ &State-action value-function  \\
		$G/D$  & Generator/Discriminator & 
		$ \mathbf{u}/ \mathbf{v}$ & Real-world image/Random noise \\
		$\rho_t$ & Aggregation parameters  &  $\x\in\mathbb{R}^m$/$\y\in\mathbb{R}^n$ & UL/LL variable \\
	   	$F/f$ & UL/LL objective &	$\S(\x)$ & Solution set of the LL subproblem (given $\x$)\\
	    $\y^*(\x)$ & BR mapping & $\widetilde{\S}(\x)$ & Solution set of the ISB subproblem (given $\x$)  \\
		$\mathtt{dist}(\cdot)$ & Point-to-set distance & $\circ$ & Compound operation \\
		$\Psi(\x)$ & $\Psi=\Psi_{T}\circ\cdots\circ\Psi_1\circ\Psi_0$ & $\Psi_{\bm{\theta}}(\x)$ & Hyper-network with parameters 
		$\bm{\theta}$ \\ 
		$\Psi_t$ &  Dynamical system at $t$-th stage&$\frac{\partial \varphi(\x)}{\partial \x}$ & Gradient of $\x$  \\
		$\frac{\partial F(\x,\y^*(\x))}{\partial\x}$ & Direct gradient of $\x$ & $\mathbf{G}(\x)$ & Indirect gradient of $\x$\\
		$\frac{\partial\y^*(\x)}{\partial\x}$ & BR Jacobian & $\frac{\partial \varphi(\x^k)}{\partial \x^k}$ & Numerical BR Jacobian w.r.t. $\x^k$\\ 
		$\varphi(\x)$ & UL value-function &$\psi(\x)$ &  LL value-function\\
		$\mathbf{d}(\mathbf{y}_{t-1};\x)$ & Optimistic aggregated gradient & $\widetilde{\mathbf{d}}(\mathbf{y}_{t-1};\x)$ & Pessimistic aggregated gradient  \\ 
		$\left(\frac{\partial^2f}{\partial\y\partial\y'}\right)^{-1}$ & Inverse Hessian matrix  &
		$\left(\frac{\partial^2 f}{\partial\y \partial\y '} 
		\right)^{-1}\frac{\partial F}{\partial \y }$ & Inverse Hessian-vector product \\	
		$t$/$k$ & Index of the LL/UL iteration & $T/K$ & Maximum LL/UL iteration number \\
		$(\cdot)_t$  &  $t$-th LL iteration   & $(\cdot)^k$ & $k$-th UL iteration  \\ 
		 $(\cdot)^{'}$ & Transposition operation& $(\x^*,\y^*)$ &The optimal UL and LL solutions\\ 
		$\mathbf{Z}_T $ & $ \sum_{t=1}^T\left(\prod_{i=t+1}^T\mathbf{A}_i\right)\mathbf{B}_t$ & $\mathbf{Z}_{T-M}$ & $  
		\sum_{t=T-M+1}^{T}\left(\prod_{i=t+1}^{T}\mathbf{A}_i\right)\mathbf{B}_{t} $ \\
		$\mathbf{P}(\y_{t-1},\bm{\omega})$ & Layer-wise transformation & $\sum\limits_{j=0}^\infty 
		\left(\mathbf{I}-\frac{\partial^2 f}{\partial\y\partial\y'}\right)^j$ & Neumann series  \\ 
		$\mathbf{A}_t$ & $\frac{\partial\Psi_{t}(\y_{t-1};\x)}{\partial \y_{t-1}}$
		& 	$ \mathbf{B}_t$ & $   \frac{\partial\Psi_{t}(\y_{t-1};\x)}{\partial\mathbf{x}}$ \\
		$\inf_{\y\in \S(\x)}F(\x,\y)$ & Optimistic objective & $\psi_{\mu}(\x)$ & parameterized LL value-function (with $\mu$)  \\
		$\sup_{\y\in \S(\x)}F(\x,\y)$ & Pessimistic objective & $\varphi_{\mu,\theta,\tau}\left(\x\right)$ & parameterized UL value-function (with $\mu$, $\theta$ and $\tau$) \\		
		\hline
	\end{tabular}
\end{table*}

\section{Fundamentals of Bi-Level Optimization}\label{sec:fundamental}


Bi-Level Optimization (BLO) contains two levels of optimization tasks, where one is nested
within the other as a constraint. The inner (or nested) and outer optimization tasks are often respectively referred to as the Lower-Level (LL) and Upper-Level (UL) subproblems~\cite{dempe2020bilevel}. Correspondingly, there are two types of variables, namely, the LL ($\y\in\mathbb{R}^n$) and UL ($\x\in\mathbb{R}^m$) variables. 
Specifically, the LL subproblem can be formulated as the following parametric optimization task
\begin{equation}
\min_{\y \in \Y} f(\x,\y), \ (\mbox{parameterized by $\x$}),\label{eq:blo-ll}
\end{equation}
where we consider a continuous function $f:\mathbb{R}^m\times \mathbb{R}^n\to\mathbb{R}$ as the LL objective and $\Y\subseteq\mathbb{R}^n$ is a nonempty set. 
By denoting the value-function as
$\psi(\x):=\min_{\y \in \Y}f(\x,\y)$,
we can define the solution set of the LL subproblem with given $\x$ as
$\S(\x):=\left\{\y\in\Y~|~f(\x,\y)\leq\psi(\x)\right\}.$
Then the standard BLO problem can be formally expressed as
\begin{equation}
\min\limits_{\x\in\X}F(\x,\y), \ s.t. \ \y\in\S(\x),\label{eq:blo}
\end{equation}
where the UL objective $F:\mathbb{R}^m\times\mathbb{R}^n\to\mathbb{R}$ is also a continuous function and the feasible set $\X\subseteq\mathbb{R}^m$. In fact, a feasible solution to BLO in Eq.~\eqref{eq:blo} should be a vector of UL and LL variables, such that it satisfies all the constraints
in Eq.~\eqref{eq:blo}, and the LL variables are optimal to
the LL subproblem in Eq.~\eqref{eq:blo-ll} for the given UL variables
as parameters. In Fig.~\ref{fig:blo}, we provide a simple visual illustration for BLOs stated in Eq.~\eqref{eq:blo}. 
\begin{figure}[ht]
	\centering  
	\subfigure[]{
		\label{Fig.sub.1}
		\includegraphics[width=0.284\textwidth]{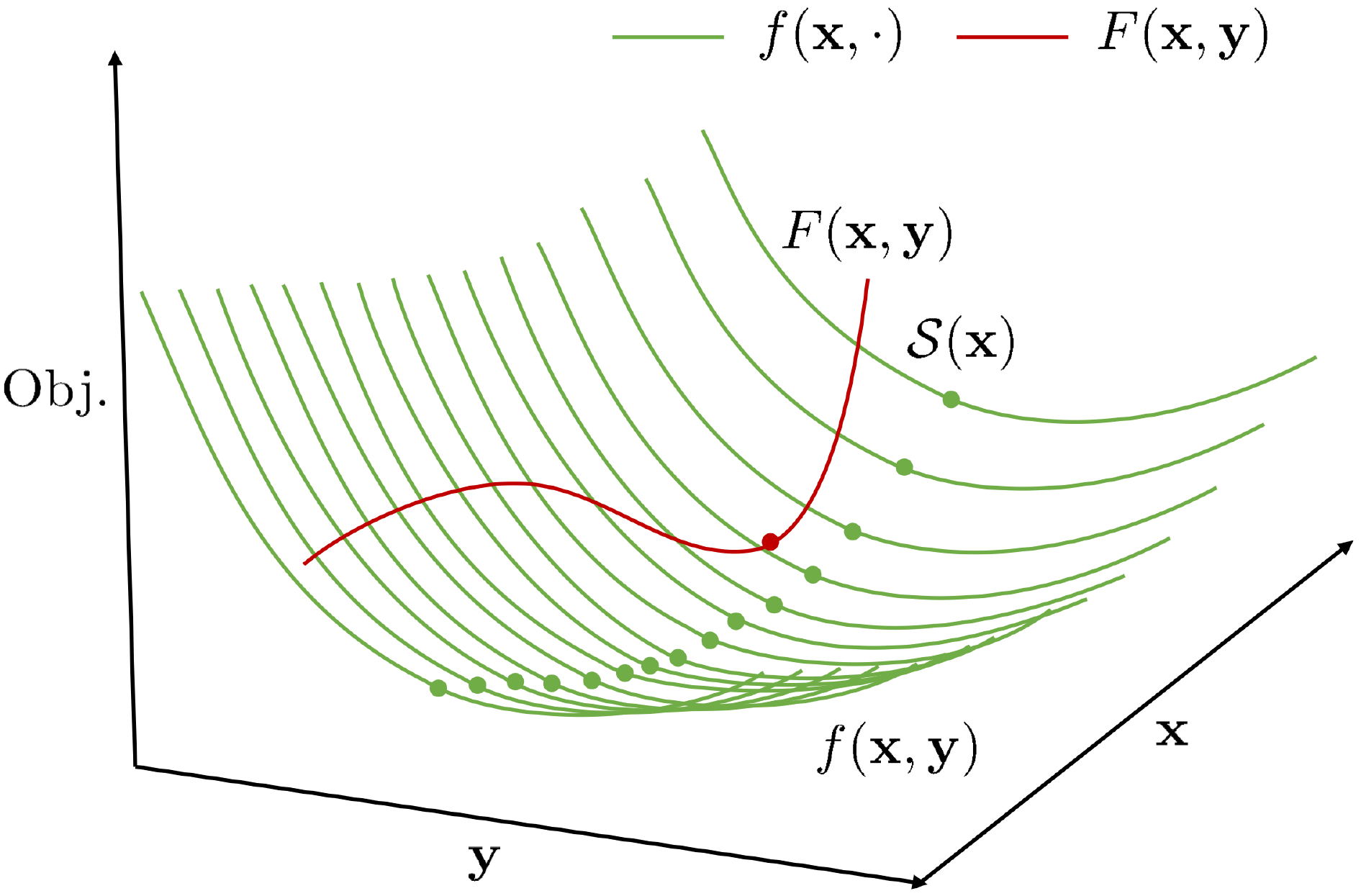}}
	\subfigure[]{
		\label{Fig.sub.2}
		\includegraphics[width=0.183\textwidth]{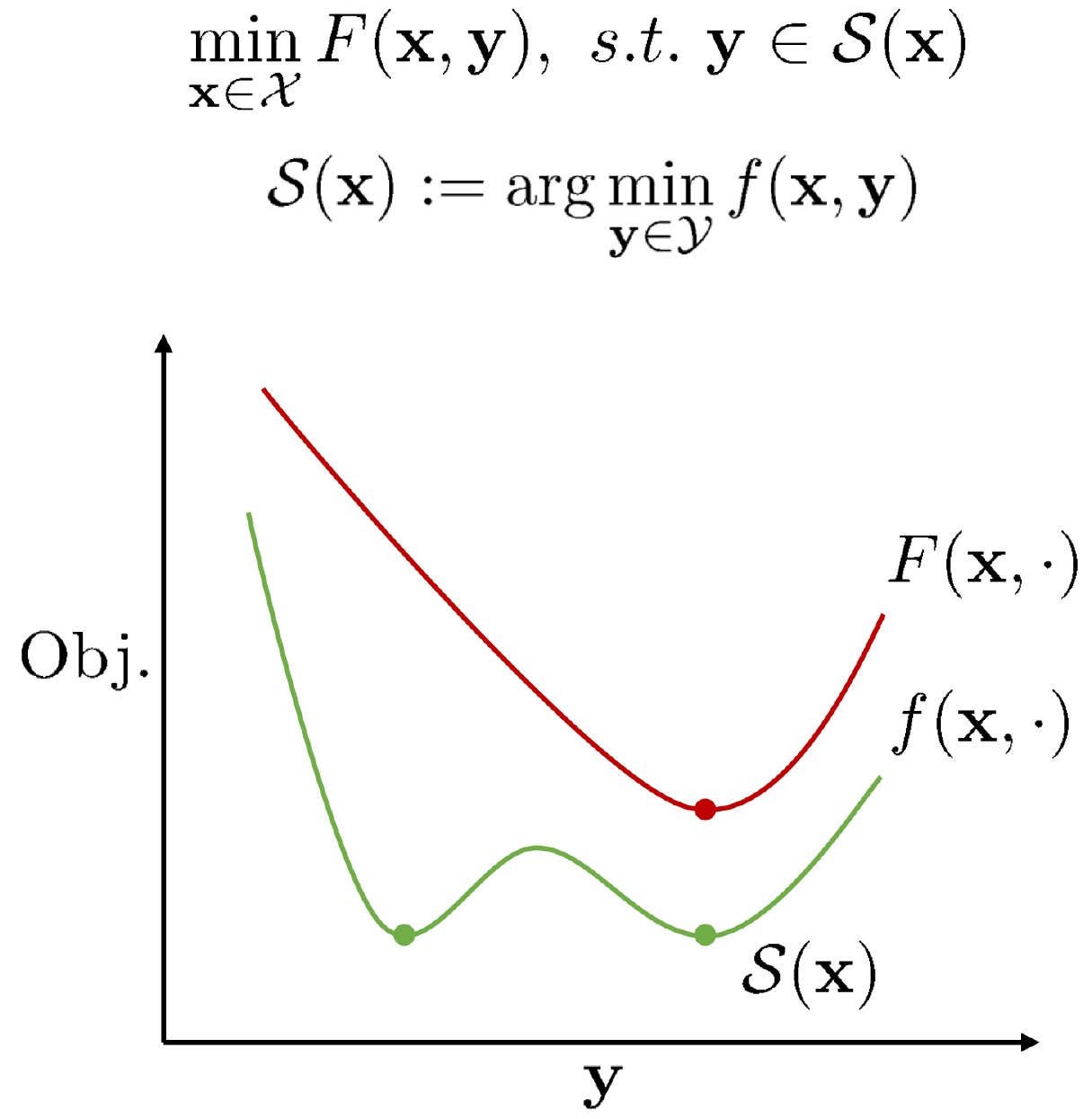}}
	\caption{Illustrating the problem of BLO. (a) first shows a standard BLO problem with the situation of multiple solutions of $f$. Green curves denote LL objectives denoted by $f$, and their corresponding minimizers given by $\S(\x)$ were shown as green dots. The red curve represents the UL objective $F$, whose minimizer is shown as the red dot. (b) further illustrates that, in general, not all points (green dots) in $\S(\x)$ could minimize the UL objective denoted by $F$.}
	\label{fig:blo}
\end{figure}

The above BLO problem has a natural interpretation as a non-cooperative game between two players (i.e., Stackelberg game~\cite{dempe2020bilevel}). Correspondingly, we may also call the UL and LL subproblems as the leader and follower, respectively. Then the ``leader'' chooses the decision $\x$ first, and afterwards the ``follower'' observes $\x$ so as to respond with a decision $\y$. Therefore, the follower may depend on the leader's decision. Likewise, the leader has to satisfy a constraint that depends on the follower’s decision.

It is worthwhile noting that the LL subproblem may have multiple solutions for every (or some) fixed value of the UL decision making variable $\x$. 
When the solution of the LL subproblem is not unique, it is difficult for the leader to predict which point in $\S(\x)$ the follower will choose (see Fig.~\ref{fig:blo} (b) for example).

\section{\textcolor[rgb]{0.00,0.00,0.00} {Understanding and Modeling Practical Problems by BLOs}}\label{sec:alltasks}

In this section, we demonstrate that even with different motivations and mechanisms, a variety of modern complex learning and vision tasks (e.g, hyper-parameter optimization, multi-task and meta learning, neural architecture search, adversarial learning, deep reinforcement learning and so on) actually share close relationships from the BLO perspective. Moreover, we provide a uniform BLO expression to (re)formulate all these problems. Table~\ref{tab:All_app} provides a summary of learning and vision applications, which can be understood and modeled by BLO. 
\begin{table*}[htb]
	\small
	\caption{Summary of related learning and vision applications that can be (re)formulated as BLOs. The abbreviations are listed as follows: Hyper-parameter Optimization (HO), Meta-Feature Learning (MFL), Meta-Initialization Learning (MIL), Neural Architecture Search (NAS), Adversarial Learning (AL), and Deep Reinforcement Learning (DRL). }\label{tab:All_app}
	\centering
	\renewcommand\arraystretch{1.2}
	\begin{tabular}{|c||c|c|}
		\hline
		
		\multirow{1}{*}{Task} & Important work & Other work \\
		\hline
		\hline
		HO&  \makecell*[l]{\hspace{1em}\cite{franceschi2017forward} (ICML, 2017), \\
			\hspace{1em}\cite{shaban2019truncated} (AISTATS, 2019),\\
			 \hspace{1em}\cite{liu2020generic} (ICML, 2020)  }
			  & \makecell*[l]{\hspace{1em}\cite{deledalle2014stein} (SIAM, 2014),  \cite{jiang2020hyper} (EURO, 2020), \cite{maclaurin2015gradient} (ICML, 2015),  \cite{lorraine2020optimizing} (AISTATS, 2020),\\
		\hspace{1em}\cite{franceschi2018far} (ICML, 2018), \cite{pedregosa2016hyperparameter} (ICML, 2016), \cite{Likhosherstov2020UFOBLOUF} (arXiv, 2019), \cite{mackay2019self} (ICLR, 2019),\\
		\hspace{1em}\cite{liu2020boml} (ICML, 2021),
		\cite{amos2017optnet} (ICML, 2017), 
		\cite{bae2020delta} (NIPS, 2020), 
		 \cite{franceschi2018bilevel} (ICML, 2018)} \\
		\hline
		\hline
		MFL & \makecell*[l]{\hspace{1em}\cite{franceschi2018bilevel} (ICML, 2018),\\
			\hspace{1em}\cite{chen2017learning} (ICML, 2017)} & \makecell*[l]{\hspace{1em}\cite{franceschi2017bridge} (arXiv, 2017), \cite{qiao2018few} (CVPR, 2018), \cite{gidaris2018dynamic} (CVPR, 2018), \cite{mishra2017simple} (ICLR, 2018), \\
		\hspace{1em}\cite{li2016learning} (ICLR, 2017), \cite{rusu2018meta} (ICLR, 2018), \cite{zintgraf2019fast} (ICML, 2019)} \\
		\hline
		\hline 
		MIL &  \makecell*[l]{\hspace{1em}\cite{rajeswaran2019meta} (NIPS, 2019),\\
			 \hspace{1em}\cite{antoniou2018train} (ICLR, 2019),\\
			\hspace{1em}\cite{bertinetto2018meta} (ICLR, 2019),\\
			\hspace{1em}\cite{park2019meta} (NIPS, 2019)}
			   & \makecell*[l]{\hspace{1em}\cite{ravi2016optimization} (ICLR, 2017),  \cite{li2017meta} (ICML, 2017), \cite{finn2017model} (ICML, 2017), \cite{nichol2018first} (arXiv, 2018),  \\
		\hspace{1em}\cite{behl2019alpha} (arXiv, 2019), \cite{guo2020learning} (CVPR, 2020),\cite{wu2020enhanced} (AAAI, 2020), \cite{nichol2018reptile} (arXiv, 2018), \\
		\hspace{1em}\cite{zhou2019efficient} (NIPS, 2019), \cite{song2019maml} (ICLR, 2020), \cite{pmlr-v97-denevi19a} (ICML, 2019),\cite{lee2018gradient} (ICML, 2018),\\
		\hspace{1em}\cite{soh2020meta} (CVPR, 2020), \cite{tian2020differentiable} (AAAI, 2020),  \cite{hsu2020meta} (ICASSP, 2020)} \\
		\hline
		\hline
		NAS & \makecell*[l]{\hspace{1em}\cite{liu2018darts} (ICLR, 2019),\\
			\hspace{1em}\cite{wong2018transfer} (NIPS, 2018), \\
			\hspace{1em}\cite{dong2020automatic} (TGRS, 2020),\\
			\hspace{1em}\cite{jiang2020sp} (CVPR, 2020),\\
			\hspace{1em}\cite{xu2019pc} (ICLR, 2019)}  &\makecell*[l]{\hspace{1em}\cite{xie2018snas} (ICLR, 2019),  \cite{pham2018efficient} (CVPR, 2018),  \cite{lian2019towards} (ICLR, 2019), \cite{chen2019progressive} (ICCV, 2019), \\
		\hspace{1em}\cite{noy2020asap} (AISTATS, 2020),  \cite{elsken2020meta} (CVPR, 2020), \cite{yao2020efficient} (AAAI, 2020), \cite{hu2020dsnas} (CVPR, 2020), \\
		\hspace{1em}\cite{liu2019auto} (CVPR, 2019), \cite{chen2019detnas} (NIPS, 2019), \cite{xu2019auto} (ICCV, 2019), \cite{chang2019data} (NIPS, 2019),  \\ 
		\hspace{1em}\cite{yu2020c2fnas} (CVPR, 2020), \cite{zhou2019auto} (arXiv, 2019), \cite{li2020autost} (SIGKDD, 2020), \cite{guo2020hit} (CVPR, 2020),\\
		\hspace{1em}\cite{he2020milenas} (CVPR, 2020), \cite{cheng2020differentiable} (arXiv, 2020)}\\ 
		\hline
		\hline 
		AL & \makecell*[l]{\hspace{1em}\cite{metz2016unrolled} (arXiv, 2016),\\ 
			\hspace{1em}\cite{pfau2016connecting} (arXiv, 2016)}
			 &\makecell*[l]{\hspace{1em}\cite{yin2020meta} (AAAI, 2020), \cite{gao2020adversarialnas} (CVPR, 2020),   \cite{jiang2018learning} (arXiv, 2018),   \cite{li2019learning} (PR, 2019),\\
		\hspace{1em}\cite{jin2020local} (ICML, 2020), \cite{liu2020gl} (CVPR, 2020), \cite{tian2020alphagan} (CVPR, 2020), \cite{hamm2018k} (ICML, 2018)}\\	
		\hline 
		\hline 
		DRL & \makecell*[l]{\hspace{1em}\cite{zhang2020bi} (AAAI, 2020), \\ 
			\hspace{1em}\cite{pfau2016connecting} (arXiv, 2016),\\
		\hspace{1em}\cite{wang2020global} (ICML, 2020)} 
		& \makecell*[l]{\hspace{1em}\cite{chen2019research} (CIRED, 2019), \cite{hong2020two} (arXiv, 2020), \cite{tschiatschek2019learner} (NIPS, 2019), \cite{zhang2020generative} (ICML, 2020),\\
			\hspace{1em}\cite{yang2020hierarchical} (AAMAS, 2020),  \cite{gao2019graphnas} (arXiv, 2019), \cite{ye2019deep} (TSG, 2019), \cite{ho2016generative} (NeurIPS, 2016), \\
			\hspace{1em}\cite{li2017infogail} (NeurIPS, 2017), \cite{torabi2018generative} (arXiv, 2018), \cite{wang2020global} (ICML, 2020), \cite{liu2019taming} (ICML, 2019) } \\
		\hline
		\hline
		Others & \makecell*[l]{\hspace{1em}\cite{kunisch2013bilevel} (SIAM, 2013), \\
			\hspace{1em}\cite{ochs2015bilevel} (SSVM, 2015), \\
			\hspace{1em}\cite{liu2020bilevel} (TIP, 2020),\\
			\hspace{1em}\cite{li2020bilevel} (TNNLS, 2020),\\
			 \hspace{1em}\cite{zhou2016bilevel} (TIP, 2016) } &  \makecell*[l]{\hspace{1em}\cite{fernando2016learning} (ICML, 2016),\cite{pfau2019spectral} (ICLR, 2019),   \cite{ijcai2020-101} (IJCAI, 2020), 
			 \cite{d2019bilevel} (arXiv, 2019),\\
			 \hspace{1em}\cite{stadie2020learning} (UAI, 2020),  \cite{learning2021deformable} (arXiv, 2021), \cite{pham2020bilevel} (arXiv, 2020), \cite{liu2020investigating} (TIP, 2020), \\
			 \hspace{1em}\cite{liu2019bilevel} (arXiv, 2019), \cite{liu2020retinex} (arXiv, 2020), \cite{stouraitis2020online} (T-RO, 2020), \cite{mounsaveng2020learning} (WACV, 2020), \\	
			 \hspace{1em}\cite{borsos2020semi} (arXiv, 2020), \cite{borsos2020coresets} (NIPS, 2020), \cite{chada2020consistency} (arXiv, 2020), \cite{yousefian2020differentiable} (arXiv, 2020),\\
			 \hspace{1em}\cite{liu2020mnemonics} (CVPR, 2020),   \cite{litany2018soseleto} (ICLR, 2018)}  \\
		\hline
	\end{tabular}
\end{table*}

\subsection{Hyper-parameter Optimization}
Hyper-parameter Optimization (HO) refers to the problem of identifying the optimal set of hyper-parameters that can’t be learned using the training data alone. 
Early in learning and vision areas, designing regularized models or support vector machines are generally the recommended approaches of selecting hyper-parameters~\cite{bennett2008bilevel}. 
Based on the representation of the hierarchical structure, 
these approaches are first expressed as a BLO problem and then transformed into the single-level optimization problem by replacing the LL subproblem with its optimality condition~\cite{couellan2015bi}. 
Due to the high computational cost, especially in high-dimensional hyper-parameter space, these original methods even could not guarantee a local optimal solution~\cite{bergstra2013making}. 

In recent years, gradient-based HO methods with deep neural networks have received extensive attention, which are generally divided into two categories: iterative differentiation (i.e., \cite{domke2012generic,baydin2014automatic,ochs2015bilevel,maclaurin2015gradient,okuno2018hyperparameter,franceschi2017bridge,franceschi2018bilevel,Likhosherstov2020UFOBLOUF,shaban2019truncated,franceschi2018far}) and implicit differentiation (i.e., \cite{chapelle2002choosing,seeger2008cross,kunisch2013bilevel,calatroni2017bilevel,mackay2019self,foo2008efficient,pedregosa2016hyperparameter,lorraine2018stochastic,lorraine2020optimizing}), depending on how the gradient (w.r.t. hyper-parameters) can be computed. 
The former approximates the best-response function by performing several steps of gradient descent on the loss function, while the latter derives the hyper-gradients through the implicit function theory. 
One particular type of gradient-based HO is the data hyper-cleaning problem~\cite{franceschi2017forward,shaban2019truncated}, which generally trains a linear classifier with a cross-entropy function (w.r.t. parameters $\y$) and learns to optimize the hyper-parameters $\x$ with a $\ell_2$ regularization function.

\begin{figure}[htb]
	\centering \begin{tabular}{c@{\extracolsep{0.2em}}c}
		\includegraphics[width=0.44\textwidth]{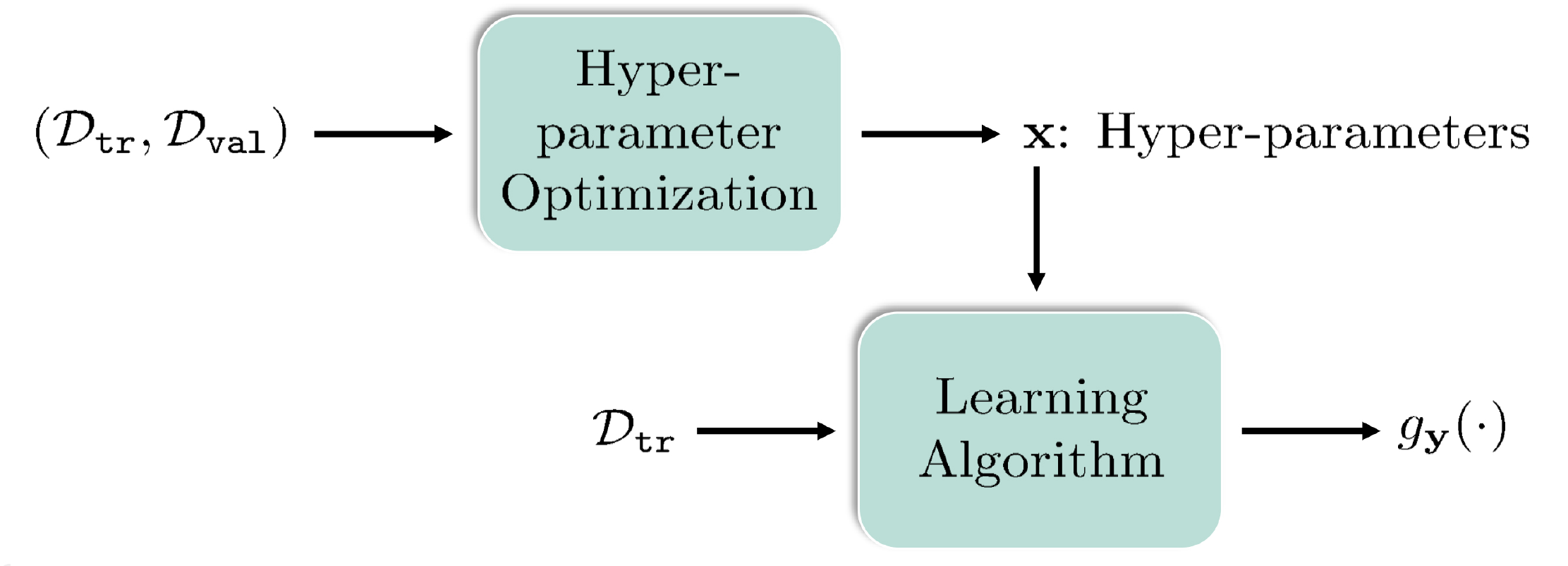}
	\end{tabular}
	\caption{Schematic diagram of HO. The UL subproblem involves optimization of hyper-parameters $\x$ based on $\left(\mathcal{D}_{\mathtt{tr}}, \mathcal{D}_{\mathtt{val}}\right)$, while the LL subproblem involves optimization of weight parameters $\y$, aiming to find the learning algorithm $g_{\mathbf{y}}(\cdot)$ based on $\mathcal{D}_{\mathtt{tr}}$.  }\label{fig:fighpo}
\end{figure}

\textcolor[rgb]{0.00,0.00,0.00} {Indeed, HO can be understood as the most straightforward application of BLO in learning and vision fields~\cite{bennett2008bilevel}}. 
Specifically, the UL objective $F(\x,\y;\mathcal{D}_{\mathtt{val}} )$ aims to minimize the validation set loss with respect to the hyper-parameters (e.g. weight decay), and the LL objective $f(\x,\y;\mathcal{D}_{\mathtt{tr}} )$ needs to output a learning algorithm by minimizing the training loss with respect to the model parameters (e.g. weights and biases). 
As illustrated in Fig.~\ref{fig:fighpo}, the full dataset $\mathcal{D}$ is divided into the training and validation datasets (i.e., $\mathcal{D}_{\mathtt{tr}}\cup \mathcal{D}_{\mathtt{val}}$) and we instantiate how to model the HO task from the perspective of BLO. 
Inspired by this nested optimization, most HO applications can be characterized by the bi-level structure and formulated as the BLO problems. The UL subproblem involves the optimization of hyper-parameters $\x$ and the LL subproblem (w.r.t. weight parameters $\y$) aims to find the learning algorithm $g_{\mathbf{y}}(\cdot)$ by minimizing the training loss.

\subsection{Multi-task and Meta Learning}

The goal of meta learning (a.k.a., learning to learn) is to design models that can learn new skills or adapt to new environments rapidly with a few training examples (see Fig.~\ref{fig:ML} for a schematic diagram). As a variant of meta learning, multi-task learning just intends to jointly perform all the given tasks~\cite{thrun1998learning,ruder2017overview}. One of the most well-known instances of meta learning is few-shot classification (i.e., $N$-way $M$-shot). Each task is a $N$-way classification designed to learn the meta-parameter with $M$ training samples selected from each of the class. 
Specially, the full meta training data set $\mathcal{D}=\{\mathcal{D}^j\}$ ($j=1,\cdots,N$) can be segmented into $\mathcal{D}^j=\mathcal{D}_{\mathtt{tr}}^j\cup\mathcal{D}_{\mathtt{val}}^j$, where $\mathcal{D}^j$ is linked to the $j$-th task. 
\begin{figure}[htb]
	\centering \begin{tabular}{c@{\extracolsep{0.2em}}c}
		\includegraphics[width=0.44\textwidth]{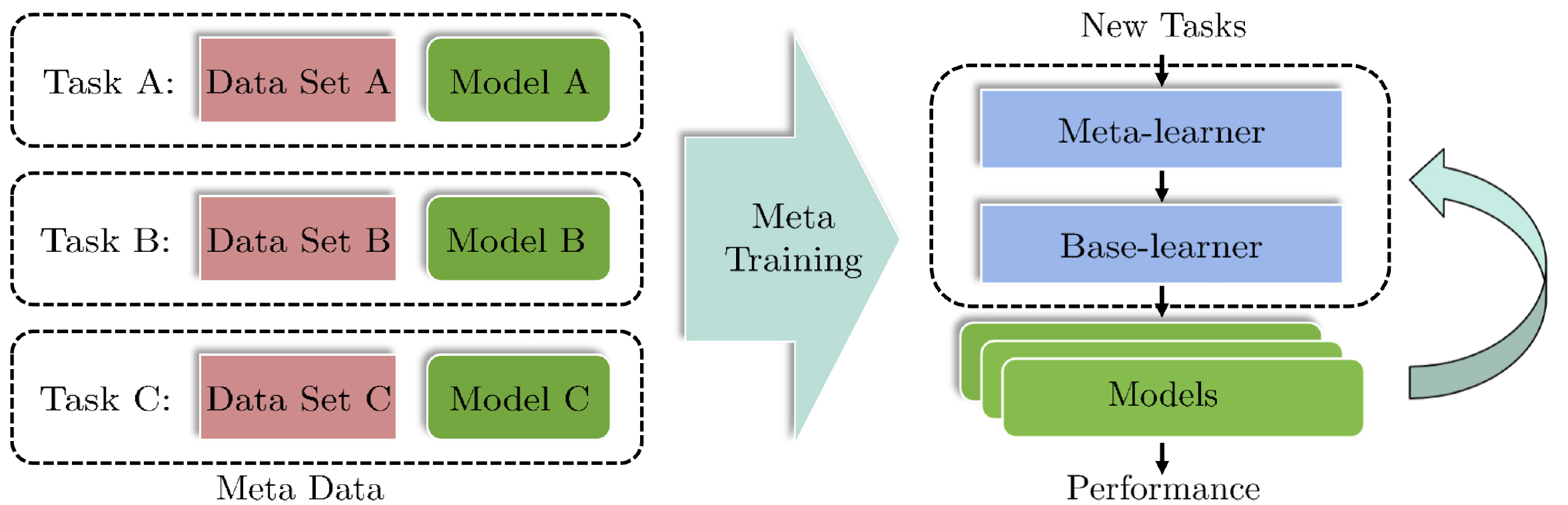}
	\end{tabular}
	\caption{Illustrating the training process of meta learning. The whole process is visualized to learn new tasks quickly by drawing upon related tasks on corresponding data sets. It can be decomposed into two parts: the ``base-learner'' trained for operating a given task and the ``meta-learner'' trained to learn how to optimize the base-learner.
	}\label{fig:ML} 
\end{figure}

According to the dependency between the meta-parameters and the network parameters, current meta learning based methods can be roughly categorized as two groups, i.e., meta-feature learning and meta-initialization learning, as can be seen in Fig.~\ref{Fig.ML}. 
\textcolor[rgb]{0.00,0.00,0.00} {Specifically, meta-initialization learning aims to investigate the meta information of multiple tasks by the network initialization, which can also be understood as the promotion of fine-tuning~\cite{rajeswaran2019meta,park2019meta}. From the BLO perspective, we actually formulate the network parameters and their initialization (based on multi-task information) by the LL and UL subproblems, respectively. In contrast, meta-feature learning methods first separate the network architecture as the meta feature extraction part and the task-specific part. Then they formulate a hierarchical learning process~\cite{franceschi2017bridge,franceschi2018bilevel,chen2017learning}. So in such tasks, we use the UL and LL subproblems to model the meta-feature part and the task-specific part, respectively.}

\subsubsection{Meta-feature Learning }

Meta-Feature Learning (MFL) aims to learn a sharing meta feature representation of all tasks. 
Recently, series of meta learning based approaches show that multi-task with hard parameter sharing and meta-feature representation are essentially similar~\cite{zhao2018data,alesiani2020towards}.
The optimization of meta-learner with respect to meta-parameters based on the UL subproblem is similar to HO~\cite{franceschi2017bridge,franceschi2018bilevel,franceschi2018far}. 
The cross-entropy function $\ell(\x,\y^j;\mathcal{D}_{\mathtt{tr}}^j)$ is actually considered as the task-specific loss for the $j$-th task on the meta training data set to define the LL objective. 

As illustrated in the subfigure (a) of Fig.~\ref{Fig.ML}, following the bi-level framework, the network architecture in this category can be subdivided into two groups. The first is the cross-task intermediate representation layer parameterized by $\x$ (illustrated by the blue block), outputting the meta features. 
The second is the logistic regression layer parameterized by $\y^j$ (illustrated by the green block), as the ground classifier for the $j$-th task. 
As can be seen, the feature layers are shared across all episodes, while the softmax regression layer is episode (task) specific.  
We can also observe that the process of network forward propagation corresponds to the process of passing from the feature extraction part to the softmax part.
\subsubsection{Meta-initialization Learning}

Meta-Initialization Learning (MIL) aims to learn a meta initialization for all tasks.
MAML~\cite{finn2017model}, known for its simplicity, estimates initialization parameters with the cross-entropy and mean-squared error for supervised classification and regression tasks purely by the gradient-based search.  
Except for initial parameters, recent approaches have focused on learning other meta variables, such as updating strategies (e.g., descent direction and learning rate~\cite{andrychowicz2016learning,ravi2016optimization,behl2019alpha}) and an extra preconditioning matrix (i.e.,~\cite{flennerhag2019meta,rusu2018meta,lee2018gradient}). 
Moreover, implicit gradient methods have a rapid development in the context of few-shot meta learning. 
There exist a large variety of algorithms replacing the gradient process of the optimization of base-learner through calculation of implicit meta gradient~\cite{rajeswaran2019meta,zhou2019efficient,balcan2019provable,bertinetto2018meta}. 
Due to the large amount of computation required to calculate the Hessian vector product in the training process, various Hessian-free algorithms have been proposed to alleviate the costly computation 
of second-order derivatives, including but not limited to~\cite{nichol2018reptile,ji2020multi,song2019maml,antoniou2018train,li2017meta,chen2017learning}. 
In particular, various first-order approximation BLO algorithms have been proposed to avoid the time-consuming calculation of second-order derivatives in~\cite{nichol2018first}. 
For instance, 
a modularized optimization library was proposed in~\cite{liu2020boml} to unify several meta learning algorithms into a common BLO framework\footnote{The code for this library is available at \url{https://github.com/dut-media-lab/BOML}.}. 

As can be shown in subfigure (b) of Fig.~\ref{Fig.ML}, $\x$ denoted by blue blocks corresponds to network initialization parameters, and $\y$ denoted by green blocks corresponds to model parameters and is treated as the updated variable satisfying the condition $\y_0^j=\x$.
Compared to MFL, there is no deeply intertwined and entangled relationship between two variables $(\x,\y^j)$, and $\x$ is only explicitly related to $\y$ in the initial state. 
As a bi-level coupled nested loop strategy, the LL subproblem based on base-learner is trained for operating a given task, and the UL subproblem based on meta-learner aims to learn how to optimize the base-learner.  
Among the well-known approaches in this direction, most recent approaches (i.e.,  \cite{lee2017meta,nichol2018first}) have claimed that the LL objective is denoted by the task-specific loss on the training data set, i.e., $f(\x,\{\y^j\})=\ell(\x,\y^j;\mathcal{D}_{\mathtt{tr}}^j)$.
By utilizing cross-entropy function, the UL objective is given by $F(\x,\{\y^j\})=\sum _{j}\ell(\x,\y^j;\mathcal{D}_{\mathtt{val}}^j).$ 

\begin{figure}[htb]
	\centering  
	\subfigure[MFL ]{
		\label{Fig.sub.ML-L}
		\includegraphics[width=0.192\textwidth]{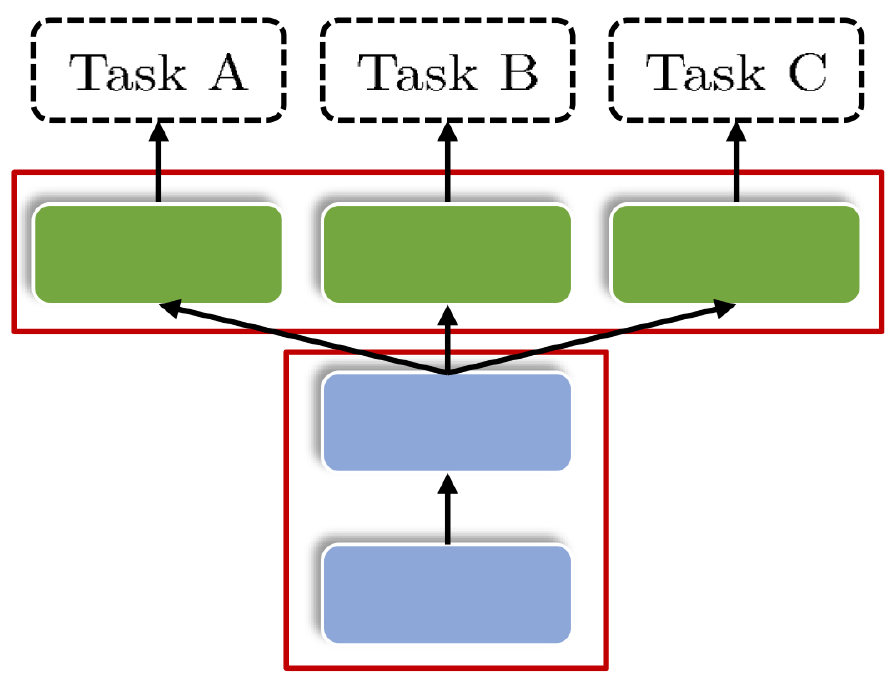}}
	\subfigure[MIL]{
		\label{Fig.sub.ML-R}
		\includegraphics[width=0.279\textwidth]{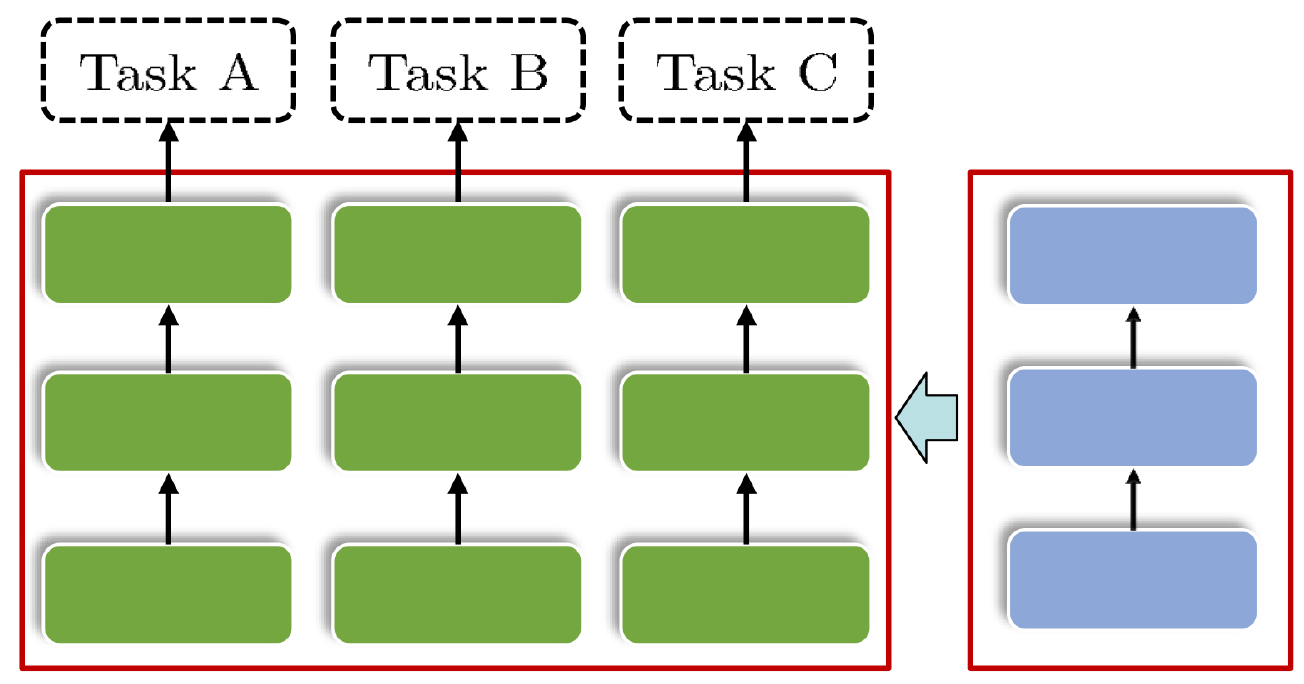}}
	\caption{Illustration of two architectures that are generally applied to multi-task and meta learning: MFL and MIL. Both of them can be separated into two parts: meta-parameters denoted by $\x$ (blue blocks) and parameters denoted by $\y^j$ (green blocks). (a) shows meta-parameters for features shared across tasks and parameters of the logistic regression layer. (b) shows meta (initial) parameters shared across tasks and parameters of the task specific layer. }
	\label{Fig.ML}
\end{figure}

Both MFL and MIL are essential solution strategies of one optimizer based on another optimizer, thus conforming to the construction of the BLO scheme. 
As a task-specific loss associated with the $j$-th task, the LL objective can be defined as $\mathbf{y}^j\in\arg\min_{\y^j\in\mathcal{Y}} f\left(\x,\y^j;\mathcal{D}_{\mathtt{tr}}^j\right)$, $j=1,\cdots,N$. 
Also, based on $\{\mathcal{D}_{\mathtt{val}}^j\}$, the UL objective can be given by $\min_{\mathbf{x}\in\mathcal{X}}F\left(\x,\{\y^j\};\{\mathcal{D}_{\mathtt{val}}^j\}\right)$.

To summarize, the UL meta-learner performs gradient descent operations and updates the meta-parameter with feedback from base-learners to extract generalized meta knowledge. Subsequently, the better meta knowledge is fed into the base-learner (i.e., the LL subproblem) as part of its model for optimizing $\y$, thereby forming an optimization cycle.

\subsection{Neural Architecture Search}
Neural Architecture Search (NAS) seeks to automate the process of choosing the optimal neural network architecture~\cite{elsken2018neural}. 
Recently, there has aroused a great deal of interest in gradient-based differentiable NAS methods~\cite{liu2018darts,casale2019probabilistic,mendoza2016towards}.  
Specifically, these gradient-based differentiable NAS methods mainly contain three main concepts: search space, search strategy and performance estimation strategy. 
As shown in Fig.~\ref{fig:NAS}, by designing an architecture search space, they generally use a certain search strategy to find the optimal network architecture. 
Such a process can be regarded as the system of optimizing the operation and connection of each node. 
\begin{figure}[htb]
	\centering \begin{tabular}{c@{\extracolsep{0.2em}}c}
		\includegraphics[width=0.46\textwidth]{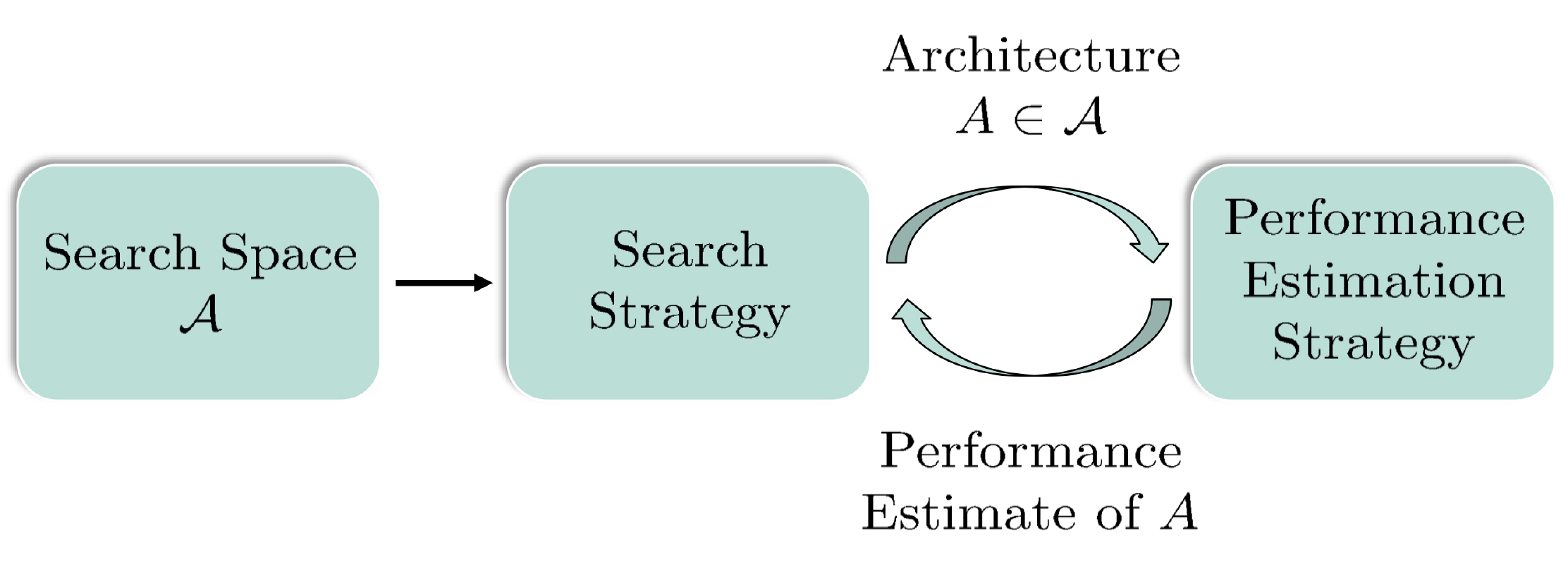}
	\end{tabular}
	\caption{Schematic diagram of NAS. Derived from a predefined search space $\mathcal{A}$, NAS first selects an architecture $A$ to transport into the performance estimation strategy, then returns the estimated performance of $A$ to the search strategy.}\label{fig:NAS}
\end{figure}

DARTS~\cite{liu2018darts}, the most well-known instance, relaxed the search space to be continuous and conducted searching for architectures in a differentiable way to simultaneously optimize the architectures and weights. 
Actually, each operation corresponds to a coefficient in DARTS.  By denoting $\mathbf{x}=\{\mathbf{x}^{ij}\}$ as the architecture parameters and $\mathbf{x}^{ij}$ as the form of connection between two nodes, the expression formula of mixed operations $\bar{o}^{ij}(\cdot)$ based on the softmax function can be written as
\begin{equation*}
\bar{o}^{ij}(\cdot)=\sum\limits_{o\in\mathcal{O}}\frac{\exp(\mathbf{x}_{o}^{ij})}
{\sum\limits_{o'\in\mathcal{O}}\exp(\mathbf{x}_{o'}^{ij})}o(\cdot),
\end{equation*}
where $o$ and $o'$ are operations and $\mathcal{O}$ is the set of all candidate operations.
Then, $o^{ij}=\arg\max_{o\in\mathcal{O}}\mathbf{x}_{o}^{ij}$ is further evaluated and performed in order to obtain the optimal architecture. 
However, due to the sharp deterioration in performance caused by the large number of skip connections, 
a great deal of improved approaches have emerged, such as ENAS~\cite{pham2018efficient}, PC-DARTS~\cite{xu2019pc}, P-DARTS~\cite{chen2019progressive}, just to name a few.

Currently, a series of gradient-based differentiable NAS methods combined with meta learning have been proposed, see~\cite{wang2020m,lian2019towards,kim2018auto,elsken2020meta}.  
Based on the bi-level coupling mechanism, these gradient-based differentiable NAS methods have achieved promising results in the numerous visual and learning applications, such as image classification~\cite{dong2020automatic}, semantic segmentation~\cite{yu2020c2fnas,zhu2019v,liu2019auto}, object detection~\cite{chen2019detnas,xu2019auto,guo2020hit,jiang2020sp,li2020autost}, medical image analysis~\cite{zhu2019v,yu2020c2fnas}, video classification~\cite{fernando2016learning}, recommendation system~\cite{cheng2020differentiable}, graph network~\cite{zhou2019auto,gao2019graphnas} and representation learning~\cite{gao2019graphnas}, etc.  

Given the proper search space, it is helpful for these gradient-based differentiable NAS methods to derive the optimal architecture for different vision and learning tasks. 
From the BLO's point of view, the UL objective w.r.t. the architecture weights (e.g. block/cell) can be parameterized by $\mathbf{x}$. And the LL objective w.r.t. the model weights can be parameterized by $\y$. 
Therefore, the full searching process can virtually be formulated as a BLO paradigm, where the UL objective is defined by $F(\x,\y;\mathcal{D}_{\mathtt{val}} )$ based on the validation data set $\mathcal{D}_{\mathtt{val}}$, and the LL objective is given by $f(\x,\y;\mathcal{D}_{\mathtt{tr}})$ based on the training data set $\mathcal{D}_{\mathtt{tr}}$.

\subsection{Adversarial Learning}

Adversarial Learning (AL) is currently deemed as one of the most important learning tasks. It has been applied in a large variety of application areas, i.e., image generation~\cite{liu2020gl,jiang2018learning,gao2020adversarialnas}, adversarial attacks~\cite{pang2020boosting} and face verification~\cite{li2019learning}. 
For example, the work proposed in~\cite{liu2020gl} introduced an adaptive BLO model for image generation, which guided the generator to reasonably modify the parameters in a complementary and promoting way. 
Moreover, a new adversarial training strategy has been proposed by learning a parametric optimizer with neural networks to study the adversarial attack~\cite{jiang2018learning}. 
As the current influential model, Generative Adversarial Network (GAN) can be deemed as deep generative models~\cite{goodfellow2014generative}. 
Recently, targeting at finding pure Nash equilibrium of generator and discriminator, the author proposed to exploit a fully differentiable search framework by formalizing as solving a bi-level mini-max optimization problem~\cite{tian2020alphagan}. 
\begin{figure}[htb]
	\centering \begin{tabular}{c@{\extracolsep{0.2em}}c}
		\includegraphics[width=0.465\textwidth]{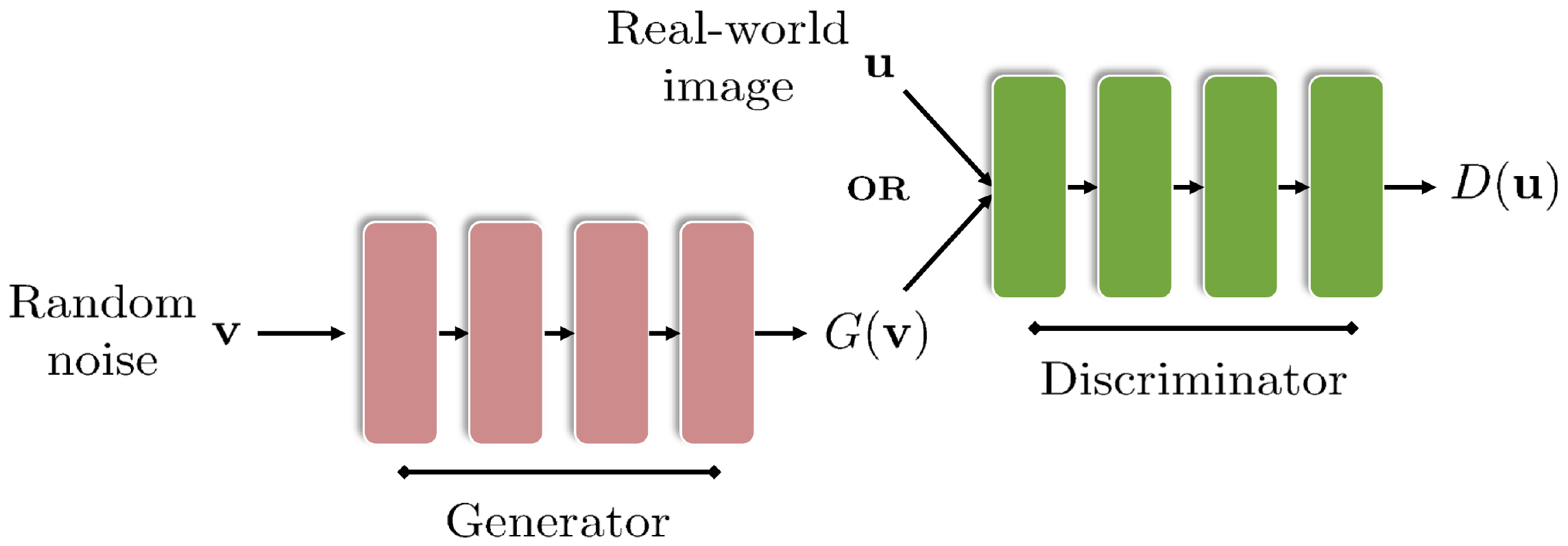}
	\end{tabular}
	\caption{ Illustrating the architecture of GAN. The generator $G$ is represented as a deterministic feed forward neural network (red blocks), through which a fixed random noise $\mathbf{v}$ is passed to output $G(\mathbf{v})$. The discriminator $D$ is another neural network (green blocks) which maps the sampled real-world image $\mathbf{u}\sim p_{data}$ and $G(\mathbf{v})$ to a binary classification probability. }\label{fig:GAN}
\end{figure}

Most of the AL approaches can formulate the unsupervised learning problem as a bi-level game between two opponents: a generator which samples from a distribution, and a discriminator which classifies the samples as real or false, as shown in Fig.~\ref{fig:GAN}. 
The goal of GAN is to minimize the duality gap denoted by $\mathcal{V}(D,G)$:
\begin{equation*} \label{eq:eqgan}
\begin{split}
\min\limits_{G}\max\limits_{D}\mathcal{V}(D,G)&=\mathbb{E}_{\mathbf{u}\sim p_{data}(\mathbf{u})}\log D(\mathbf{u})\\
&+\mathbb{E}_{\mathbf{v}\sim \mathcal{N}_{(0,1)}}\log (1-D(G(\mathbf{v}))),
\end{split}
\end{equation*}
where the fixed random noise source $\mathbf{v}$ obtained from $\mathbf{v}\sim\mathcal{N}_{(0,1)}$ is input into the generator $G$, which, together with the sampled real-world image $\mathbf{u}\sim p_{data}$, is then authenticated by the discriminator $D$. Notice that $\mathbb{E}$ denotes the expectation which implies that the average value of some functions under a probability distribution.    

Indeed, AL problems generally correspond to the mini-max BLO problems, where the UL discriminator denoted by $F$ targets on learning a robust classifier, and the LL generator denoted by $f$ tries to generate the adversarial samples. Specifically, the UL and LL objectives can be respectively formulated as 
\begin{equation*} \label{eq:eqgan-F}
\begin{split}
F(\mathbf{x},\mathbf{y})&=-\mathbb{E}_{\mathbf{u}\sim p_{data}(\mathbf{u})}\log D(\mathbf{u})\\
&-\mathbb{E}_{\mathbf{v}\sim \mathcal{N}_{(0,1)}}\log (1-D(G(\mathbf{v}))),
\end{split}
\end{equation*}
\begin{equation*}
f(\mathbf{x},\mathbf{y})=-\mathbb{E}_{\mathbf{v}\sim \mathcal{N}_{(0,1)}}\log (D(G(\mathbf{v}))),
\label{eq:eqgan-f}
\end{equation*}
where $G$ and $D$ are parameterized with variables $\y$ and $\x$, respectively. In other words, the UL subproblem aims to reduce the duality gap $\mathcal{V}(D,G)$ and the LL subproblem interactively optimizes the discriminator parameters denoted by $\x$ to obtain the optimal solution.

\subsection{Deep Reinforcement Learning}

Deep Reinforcement Learning (DRL) aims to make optimal decisions by interacting with the environment and learning from the experiences.
\textcolor[rgb]{0.00,0.00,0.00} {Indeed, a variety of DRL tasks, including 
Single-Agent Reinforcement Learning (SARL)~\cite{pfau2016connecting,yang2019provably,hong2020two}, Multi-Agent Reinforcement Learning (MARL)~\cite{chen2019research,yang2020hierarchical,zhang2020bi,wang2019bi}, Meta Reinforcement Learning (MRL)~\cite{wang2020global,liu2019taming,xu2018meta,clavera2018model}, and Imitation Learning (IL)~\cite{ho2016generative,torabi2018generative,li2017infogail}, which all can be modeled and tackled by BLO techniques. }

As for SARL problems, Actor-Critic (AC) type methods have been widely studied and viewed as a bi-level or two-time-scale optimization problems~\cite{hong2020two,yang2019provably},  as illustrated in Fig.~\ref{fig:AC}.
Indeed, AC type DRL methods often aim to simultaneously learn a state-action value-function $Q^{\pi}$ that predicts to expect the discounted cumulative reward and a policy which is optimal for that value function: 
\begin{equation*}
	\resizebox{0.95\hsize}{!}{$Q^{\pi}(s,a)=\mathbb{E}_{s_{i+j}\sim\mathcal{P},r_{i+j}\sim\mathcal{R},a_{i+j}\sim\pi}
		\left(\sum\limits_{k=0}^{\infty}\gamma^jr_{i+j}|s_i=s,a_i=a\right)$},\label{eq:eqac}
\end{equation*}
where $\mathcal{P}$ and $\mathcal{R}$ denote dynamics of the environment and reward function, $s$ and $a$ are the state and action, $i$ and $j$ represent the i-th and j-th steps, and $\mathbb{E}$ is the expectation which implies that the average value of some function under a probability distribution.  The policy maximizes the expected discounted cumulative reward for that state-action value-function, i.e.,
$
\pi^*=\arg\max_{\pi}\mathbb{E}_{s_0\sim p_0, a_0\sim \pi}\left(Q^{\pi}(s_0,a_0)\right),\label{eq:eqac1}
$
where $s_0$, $a_0$ and $p_0$ correspond to the initial state, initial action and the initial state distribution, respectively. Under the BLO paradigm, the actor and critic correspond to the UL and LL variables, respectively.
Let $\x$ denote the parameters of the state-action value-function and $\y$ denote the parameters of the policy $\pi$. The UL and LL objectives respectively take the form
\begin{equation*} \label{eq:eqac-F}
	\begin{aligned}
		F(\x,\y)  & = \mathbb{E}{s_{i},a_{i}\sim\pi}(\mathtt{div}(\mathbb{E}_{s_{i+1},a_{i+1},r_{i+1}}  \\
		& \left(r_{i+1}+\gamma Q(s_{i+1},a_{i+1})\right) \parallel Q(s_{i},a_{i}))),
	\end{aligned}
\end{equation*}
$$
	f(\x,\y) =-\mathbb{E}_{s_{0}\sim p_{0},a_{0}\sim\pi} Q^{\pi}(s_{0},a_{0}),
	\label{eq:eqac-f}
$$
where $\mathtt{div}(\cdot||\cdot)$ represents any divergence.

\textcolor[rgb]{0.00,0.00,0.00} {MARL studies how multiple agents can collectively learn, collaborate, and interact with each other in an environment. In the classical MARL system, agents are treated equally and the goal is to solve the Markov game to an arbitrary Nash equilibrium when multiple equilibria exist, thus lacking a solution for selection. To address this issue, the work in~\cite{zhang2020bi} formulates MARL as the  multi-state model-free Stackelberg equilibrium learning problem. Thus, under Markov games, they construct a BLO formulation to find Stackelberg equilibrium to address the MARL task. Similarly, a multi-agent bi-level cooperative reinforcement learning algorithm was proposed in~\cite{chen2019research} to solve the stochastic decision-making problem.}

\textcolor[rgb]{0.00,0.00,0.00} {
In recent years, MRL approaches (a.k.a., meta learning on reinforcement learning tasks), which
aim to learn a policy that adapts fast to new tasks and/or environments, have achieved remarkable success~\cite{wang2016learning,xu2018meta}. For example,
the work in~\cite{clavera2018model} learns a policy that can quickly adapt to other related models only with one policy gradient step. By adding control variables into gradient estimation, the work in~\cite{liu2019taming} can obtain low variance estimates for policy gradients. While the work in~\cite{wang2020global} characterizes the optimality gap of the stationary points attained by MAML for both reinforcement learning and supervised learning. Since all these works are based on the meta-initialization platform, it is also nature to formulate these meta reinforcement learning methods from the perspective of BLOs.
}


\textcolor[rgb]{0.00,0.00,0.00} {
Generally, IL techniques are very useful when it is easier for an expert to demonstrate the desired behavior rather than to specify a reward function which would generate the same behavior or to directly learn the policy in DRL tasks~\cite{arora2020provable}. In recent years, by connecting imitation learning with generative adversarial learning, a series of Generative Adversarial Imitation Learning (GAIL) techniques~\cite{ho2016generative,torabi2018generative,li2017infogail} have been investigated to imitate an expert in a model-free DRL scenario. Since GAIL type methods have a natural connection to the mechanism of GANs, we can definitely formulate these models using BLOs. 
}


\begin{figure}[htb]
	\centering \begin{tabular}{c@{\extracolsep{0.2em}}c}
		\includegraphics[width=0.474\textwidth]{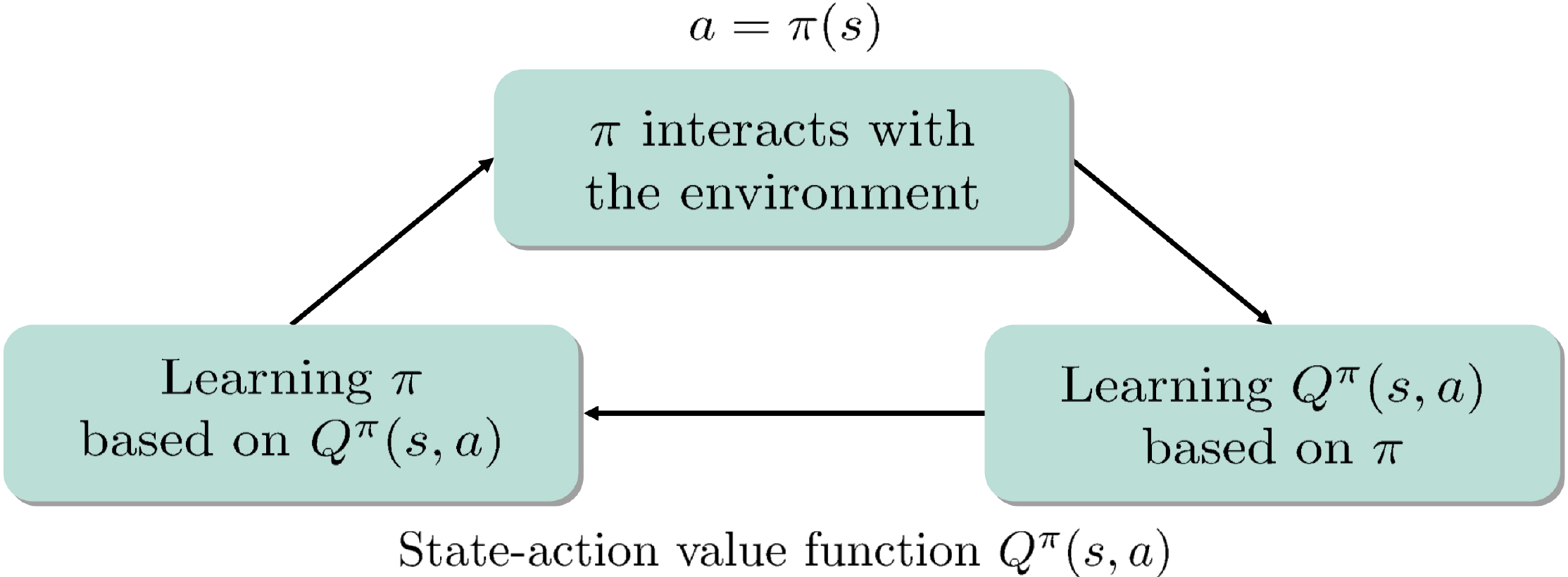}
	\end{tabular}
	\caption{ Illustrating the schematic diagram of AC learning. First the actor $\pi$ interacts with the environment to learn the state-action value-function $Q^{\pi}(s,a)$, and then the actor $\pi$ is again obtained based on $Q^{\pi}(s,a)$.}\label{fig:AC}
\end{figure}


\subsection{Other Related Applications}

The rapid development of deep learning has claimed its domination in the area of image processing and analysis. 
In addition to the above mentioned tasks, there exist a significant amount of other related learning and vision tasks that can be re(formulated) as BLO problems, such as image enhancement~\cite{kunisch2013bilevel,chen2020flexible,zhou2016bilevel,Pan2019Learning,d2019bilevel}, image registration~\cite{ijcai2020-101}, image-to-image translation~\cite{ma2019novel}, image recognition~\cite{xu2018bilevel}, image compression~\cite{zha2012hierarchical} and other related works \cite{stadie2020learning,borsos2020coresets,pfau2019spectral}. 
For example, the earlier work presented in~\cite{kunisch2013bilevel} considered the problem of parameter learning for image denoising models and incorporated $p$-norm–based analysis priors. 
Under a BLO formulation, the LL subproblem was given by a variational model which consisted of the data fidelity and regularization term, and the UL subproblem was expressed by the loss function. 
Furthermore, the work proposed in~\cite{zhou2016bilevel} formulated the discriminant dictionary learning method for image recognition tasks as a BLO. 
From this point of view, the UL subproblem can directly minimize the classification error, while the LL subproblem can use the sparsity term and the Laplacian term to characterize the intrinsic data structure. 

By addressing a unified BLO problem, the LL subproblem is usually expressed as fundamental models that conform to the laws or principles of physics, while the UL subproblem usually considers the further constraints on variables~\cite{ochs2015bilevel,fernando2016learning}.

\section{Gradient-based BLOs}\label{sec:framework}
\textcolor[rgb]{0.00,0.00,0.00} {In past years, gradient-based techniques have became the most popular BLO solution strategies in learning and vision fields. In fact, one of the first gradient-based BLO methodology is~\cite{vicente1994descent}. Currently, a variety of explicit gradient-based methods have been investigated to solve BLOs~\cite{maclaurin2015gradient,franceschi2018far,liu2021generic}. Specifically, the works in~\cite{franceschi2018bilevel,franceschi2017forward} first calculate gradient flow of the LL objective and then perform either reverse or forward gradient computations for the UL subproblem. Similar ideas have also been considered in~\cite{jenni2018deep,ochs2015bilevel,Likhosherstov2020UFOBLOUF}, but with different specific implementations. On the other hand, there also exist some implicit gradient based methods~\cite{lorraine2020optimizing,rajeswaran2019meta,grazzi2020iteration} to use the implicit function theorem to obtain the gradient.} In this section, we first review three categories of mainstream BLO formulations, which have been considered in various application scenarios. We then demonstrate how to uniformly reformulate these different BLOs from a single-level optimization perspective and investigate the intrinsic structures of existing gradient-based BLO algorithms within a unified algorithmic platform. 



\subsection{Different Formulations of BLO}

It is worthwhile to notice that the original BLO model given in Eq.~\eqref{eq:blo} is not clear in case of the multiple LL optimal solutions for some of the selections of the UL decision maker~\cite{dempe2020bilevel}. Therefore, it is necessary to define, which solution out of the multiple LL solutions in $\S(\x)$ should be considered. Here we actually consider three categories of viewpoints, i.e., singleton, optimistic and pessimistic BLOs.

The most straightforward idea in existing learning and vision literature is to assume that $\S(\x)$ is a singleton. Formally, we call the BLO model is with the Lower-Level Singleton (LLS) condition if $\forall \x\in\X$, the solution set of the LL subproblem (i.e., $\S(\x)$) is a singleton.
In this case, we can simplify the original model as
\begin{equation}
\min\limits_{\x\in\X}F(\x,\y), \ s.t. \ \y=\arg\min\limits_{\y\in\Y}f(\x,\y). \label{eq:blo-lls}
\end{equation}  
Such singleton version of BLOs is well-defined and could cover a variety of learning and vision tasks (e.g., \cite{domke2012generic,maclaurin2015gradient,pedregosa2016hyperparameter,mackay2019self}, just name a few). Thus, in recent years, dozens of methods have been developed to address this nested optimization task in different application scenarios (see the following sections for more details). 

Furthermore, the situation becomes more intricate if the LL subproblem is not uniquely solvable for each $\x\in\X$. Essentially, if 
the follower can be motivated to select an optimal solution in $\S(\x)$ that is also best for the leader (i.e., with respect to $F$), it yields the so-called optimistic (strong) formulation of BLO
\begin{equation}
\min\limits_{\x\in\X}\left\{\min\limits_{\y\in\Y}F(\x,\y), \ s.t. \ \y\in\arg\min\limits_{\y\in\Y}f(\x,\y)\right\}. \label{eq:blo-o}
\end{equation}
The above stated optimistic viewpoint has drawn increasing attention in BLO literature~\cite{dempe2007new,kohli2012optimality,lampariello2020numerically} and recently also been investigated in learning and vision fields~\cite{liu2020generic,liu2021generic,liu2021value}. In Section \ref{sec:beyond}, we will further explore  how to solve such optimistic BLOs in detail. 

If the leader does not have the information whether the follower returns the best response $\y$ from $\S(\x)$ in terms of the UL objective $F$, then we have to assume that the follower is not cooperate with the leader. This is known as the pessimistic (weak) formulation of BLO~\cite{wiesemann2013pessimistic,loridan1996weak} and can be given as: 
\begin{equation}
\min\limits_{\x\in\X}\left\{\max\limits_{\y\in\Y}F(\x,\y), \ s.t. \ \y\in\arg\min\limits_{\y\in\Y}f(\x,\y)\right\}. \label{eq:blo-p-1}
\end{equation}
It should be pointed out that till now we still lack efficient gradient-based algorithms to address the pessimistic BLO problems\footnote{In Section~\ref{sec:pessimistic}, we will demonstrate that we can also obtain some practical gradient-based iteration scheme within our general algorithmic platform for the pessimistic formulation of BLO.}.

\subsection{BR-based Single-Level Reformulation}

In this work, we consider the optimal solution of the LL subproblem with a given UL variable $\x$ as the Best-Response (BR) of the follower (denoted as $\y^*(\x)$). Then we can interpret BLO as a game process, in which the leader $\x$ considers what BR of the follower $\y$ is, i.e., how it will respond once it has observed the quantity of the leader~\cite{kicsiny2014backward,dempe2020bilevel}.  Based on the above understanding, we can reformulate the three different categories of BLOs as a unified single-level optimization problem.

\textcolor[rgb]{0.00,0.00,0.00} {
Specifically, given the UL variable $\x$, we denote the corresponding BR mapping as $\y^*(\x)$. In fact, if considering the singleton BLO, $\y^*(\x)$ can be directly obtained by the unique LL solution.
While for the optimistic and pessimistic BLOs, we actually first define their Inner Simple Bi-level (ISB) subproblems (w.r.t., $\y$)\footnote{It is known that the simple bi-level optimization is just a specific BLO problem with only one variable~\cite{liu2020generic,dutta2020algorithms}.} as
\begin{equation}\small
	\mbox{Optimistic ISB:}  \min\limits_{\y\in\S(\x)}F(\x,\y) \ \mbox{and} \ \mbox{Pessimistic ISB:} \max\limits_{\y\in\S(\x)}F(\x,\y).\label{eq:optimis-y}
\end{equation}
Then by defining the solution set of ISB as $\widetilde{\S}(\x)$, we could consider any $\y^*(\x)\in\widetilde{\S}(\x)$ as the BR mapping, because points in $\widetilde{\S}(\x)$ all obtain the minimum/maximum of $F(\x,\y)$ in $\S(\x)$. Therefore, we can formulate the general BR mapping for different categories of BLOs as follows:
\begin{equation}
	\left\{\begin{array}{ll}
		\y^*(\x):=\arg\min\limits_{\y\in\Y}f(\x,\y), \quad \mbox{Singleton},\\
		\y^*(\x)\in \widetilde{\S}(\x):= \left\{\begin{array}{l}
			\arg\min\limits_{\y\in\S(\x)}F(\x,\y), \quad \mbox{Optimistic},\\
			\arg\max\limits_{\y\in\S(\x)}F(\x,\y), \quad \mbox{Pessimistic}.
		\end{array}\right.
	\end{array}\right.
	\label{eq:best-response}
\end{equation}
Based on Eq.~\eqref{eq:best-response}, we actually obtain the following value-function-based reformulation (a single-level optimization model) for BLOs stated in Eq.~\eqref{eq:blo}, i.e.,
\begin{equation}
\min\limits_{\x\in\X}\varphi(\x):=F(\x,\y^*(\x)),\label{eq:sl-blo}
\end{equation}
in which $\varphi(\x)$ actually can be used to uniformly represent the UL value-function of $F$ from the singleton, optimistic (i.e.,  $\inf_{\y\in \S(\x)}F(\x,\y)$) and pessimistic (i.e.,  $\sup_{\y\in \S(\x)}F(\x,\y)$) viewpoints. }


\subsection{A Unified Platform for Gradient-based BLOs}

Moving one step forward, the gradient of $\varphi$ w.r.t. the UL variable $\x$ can be written as\footnote{Please notice that we actually do not distinguish between the operation of the derivatives and partial derivatives to simplify our presentation.} 
\begin{equation}
\underbrace{\frac{\partial \varphi(\x)}{\partial \x}}_{\text{grad. of $\x$}}=
\underbrace{\frac{\partial F(\x,\y^*(\x))}{\partial\x}}_{\text{direct grad. of $\x$}} \quad +
\underbrace{\mathbf{G}(\x),}_{\text{indirect grad. of $\x$}}\label{eq:blo-gradient1}
\end{equation}
where the indirect gradient $G(\x)$ can be further specified as the following two components:
\begin{equation}
\mathbf{G}(\x)=\underbrace{\overbrace{\left(\frac{\partial\y^*(\x)}{\partial\x'}\right)'}^{\text{BR Jacobian}}\overbrace{\frac{\partial F(\x,\y^*(\x))}{\partial\y}.}^{\text{direct grad. of $\y$}}}_{\text{indirect grad. of $\x$}} \label{eq:blo-gradient2}
\end{equation}  
Here we use ``grad.'' as the abbreviation of gradient and denote the transpose operation as $(\cdot)'$. Note that, $\y^*(\x)$ as a general mapping, can be given specific constraints and necessary assumptions to fit their particular requirements for these specific gradient-based BLO approaches in order to obtain different iteration formats and theoretical properties. For details, please refer to the following contents.
In fact, by simple computation, the direct gradient is easy to obtain. However, the indirect gradient is intractable to obtain because we must compute the changing rate of
the optimal LL solution with respect to the UL variable (i.e., the BR Jacobian $\frac{\partial\y^*(\x)}{\partial\x}$). Please notice that we will also call $\frac{\partial \varphi(\x^k)}{\partial \x^k}$ as the practical BR Jacobian w.r.t. $\x^k$ in the following statements. 
The computation of the indirect gradient $\mathbf{G}(\x)$ naturally motives formulating $\y^*(\x)$ and hence $\frac{\partial\y^*(\x)}{\partial\x}$. 
For this purpose, a series of techniques have recently been developed from either explicit or implicit perspectives, which  obtain their optimal solutions by recurrent differentiation through dynamic system and based on implicit differentiation theory, respectively. 

Now we demonstrate how to formulate various existing gradient-based BLOs from a unified algorithmic platform. We first summarize a general BLO updating scheme in Alg.~\ref{alg:g-blo}. It can be seen that the key component of this algorithm is to calculate the BR Jacobian. Then with $\frac{\partial \varphi(\x^k)}{\partial \x^k}$, we can just perform standard (stochastic) gradient descent schemes to update $\x^k$. Based upon our general algorithmic platform, we can observe that the main differences of these existing BLO approaches are just their specific strategies for calculating Jacobian of the BR mapping under different conditions (i.e., w/ LLS and w/o LLS).

\begin{algorithm}
	\caption{A General Gradient-based BLO Scheme}\label{alg:g-blo}
	\begin{algorithmic}[1]
		\REQUIRE The UL and LL initialization.
		\ENSURE The optimal UL and LL solutions.
		\FOR{$k=1,\cdots,K$ }
		\STATE Calculate the BR Jacobian $\frac{\partial \varphi(\x^k)}{\partial \x^k}$.
		\\ \% (Mainstream calculation strategies are summarized in Figs.~\ref{fig:blp_fun-1}-\ref{fig:blp_fun} and thoroughly surveyed in the following sections)
		\STATE Perform (stochastic) gradient descent to update $\x^k$. \\ \% (based on $\frac{\partial \varphi(\x^k)}{\partial \x^k}$)
		\ENDFOR
	\end{algorithmic}  
\end{algorithm}
	
In Fig.~\ref{fig:blp_fun-1}, we summarize mainstream gradient-based BLOs and illustrate their intrinsic relationships within our general algorithmic platform. It can be observed that in the LLS scenario, from the BR-based perspective, existing gradient methods can be categorized as two groups: Explicit Gradient for Best-Response (EGBR, stated in Section~\ref{sec:egbr}) and Implicit Gradient for Best-Response (IGBR, stated in Section~\ref{sec:igbr}). As for EGBR, there are mainly three types of methods, namely, recurrence-based EGBR  (e.g., \cite{franceschi2017forward,maclaurin2015gradient,franceschi2018bilevel,shaban2019truncated,liu2018darts}), initialization-based EGBR (e.g., \cite{nichol2018first,nichol2018reptile} ) and proxy-based EGBR methods (e.g., \cite{li2016learning,lee2017meta,park2019meta,flennerhag2019meta}), differing from each other in the way of formulating the BR mapping. For IGBR, existing works consider two groups of techniques (e.g., linear system~\cite{pedregosa2016hyperparameter,rajeswaran2019meta} and Neumann series~\cite{lorraine2020optimizing}) to alleviate the computational complexity issue for the BR Jacobian. We emphasize that the validity of above BLO methodologies must depend on the singleton of their LL solution set. 
When solving BLOs without the LLS assumption, recent works in~\cite{liu2020generic,liu2021generic} have demonstrated that we need to first construct BR mapping based on both UL and LL subproblems, and then solve two optimization subproblems, namely, the single-level optimization subproblem (w.r.t. $\x$) and the ISB subproblem (w.r.t. $\y$). While the work in \cite{liu2021value} has introduced a series of barrier functions and utilized interior point methods to obtain the BR mapping for each given $\x$. 

To end up this section, we plot Fig.~\ref{fig:blp_fun} to illustrate the optimization processes of existing mainstream gradient-based BLO methods from the BR mapping perspective and within our unified algorithmic platform. In the following (i.e., Sections~\ref{sec:egbr}-\ref{sec:beyond}), we will thoroughly survey these different categories of gradient-based BLO algorithms (including their acceleration, simplification and extension techniques) and their theoretical properties (convergence behaviors and computational complexity), accordingly.

\begin{figure*}[t]
	\centering \begin{tabular}{c@{\extracolsep{0.2em}}c@{\extracolsep{0.2em}}c@{\extracolsep{0.2em}}c}
		\includegraphics[width=0.978\textwidth]{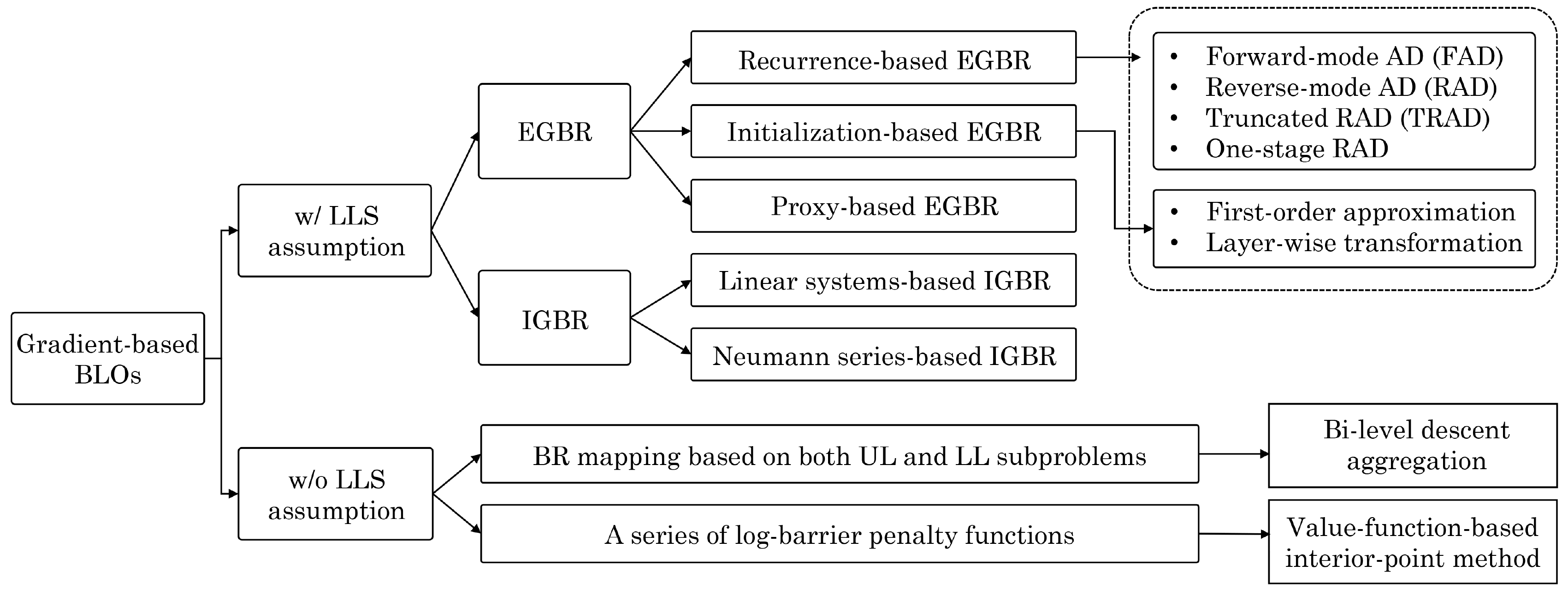}
	\end{tabular}
	\caption{Summary of the mainstream gradient-based BLOs. We categorize these existing approaches into two main groups, i.e., w/ and w/o LLS assumptions. When solving BLOs with LLS assumption, these methods can be further divided into two categories: EGBR and IGBR. As for EGBRs, they can be solved by different Automatic Differentiation (AD) techniques (as denoted by the dashed rectangle). Very recently, two algorithms have also been proposed to address BLOs without the LLS assumption. In particular, they actually introduce a bi-level gradient aggregation or a value-function-based interior-point method to calculate the indirect gradient. } \label{fig:blp_fun-1}
\end{figure*}

\begin{figure*}[t]
	\centering \begin{tabular}{c@{\extracolsep{0.2em}}c@{\extracolsep{0.2em}}c@{\extracolsep{0.2em}}c}
		\includegraphics[width=0.98\textwidth]{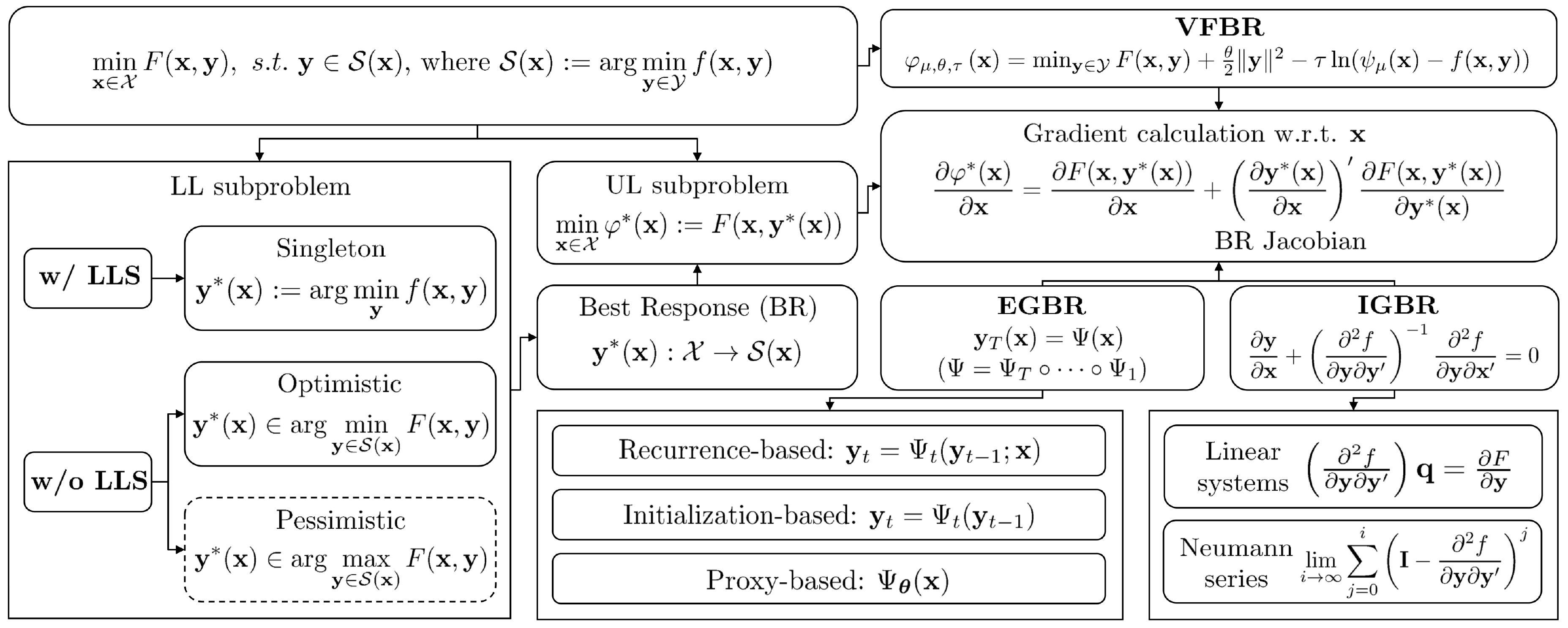}
	\end{tabular}
	\caption{Illustrating the roadmap of different categories of gradient-based BLOs. In the left bottom region, the formulations in the solid rectangles (i.e., singleton and optimistic) have been widely studied. In contrast, since gradient-based methods for pessimistic BLOs have not been properly investigated in existing literature, we denote this category of formulation by a dashed rectangle.  In Section~\ref{sec:pessimistic}, we demonstrate that we can also obtain a practical pessimistic BLO scheme within our general algorithmic platform.}\label{fig:blp_fun}
\end{figure*}

\section{Explicit Gradient for Best-Response}\label{sec:egbr}
With the LLS condition, we delve deep into the EGBR category of methods, which aims to perform automatic differentiation through the LL dynamic system~\cite{weinan2017proposal,lu2018beyond} to solve the BLO problem. Specifically, 
given an initialization $\y_0=\Psi_0(\x)$ at $t=0$, the iteration process of EGBRs can be generally written as 
\begin{equation}
\y_{t}=\Psi_t(\y_{t-1};\x), \ t=1,\cdots, T,\label{eq:ds}
\end{equation}
where $\Psi_t$ denotes some given updating scheme (based on the LL subproblem) at $t$-th stage and $T$ denotes the overall LL iterations number. 
For example, we can formulate $\Psi_t$ based on the gradient descent rule, i.e.,
\begin{equation}
\Psi_t(\y_{t-1};\x)=\y_{t-1}-\eta_t\mathbf{d}_f(\y_{t-1},\x),\label{eq:psi-gradient-descent}
\end{equation}
where $\mathbf{d}_f(\y_{t-1},\x)$ is the descent mapping of $f$ at $t$-th stage (e.g., $\mathbf{d}_f(\y_{t-1},\x)=\frac{\partial f(\x,\y_{t-1})}{\partial \y_{t-1}}$) and $\eta_t$ denotes the corresponding step size . 
Then we can calculate  $\frac{\partial \varphi(\x^k)}{\partial \x^k}$  by substituting $\y_T:=\Psi(\x)$ approximately for $\y^*(\x)$, and the full dynamical system can be defined as 
\begin{equation}
\Psi(\x):=\Psi_T\circ\cdots\circ\Psi_{1}\circ\Psi_{0}(\x). \label{eq:full-dynam}
\end{equation}
Here the notation $\circ$ represents the compound dynamical operation of the entire iteration.
That is, we actually consider the following optimization model
\begin{equation}
\min\limits_{\x\in\X}\varphi_T(\x):=F(\x,\y_T(\x)),\label{eq:blo-T}
\end{equation}
and need to calculate $\frac{\partial\varphi_T(\x)}{\partial\x}$ (instead of Eq.~\eqref{eq:blo-gradient1}) in the practical optimization scenario. 
Since it should be noted that $\Psi$ actually obtains an explicit gradient for best-response of the follower, we call this category of gradient-based BLOs as EGBR approaches hereafter. 
Starting from the Eq.~\eqref{eq:ds}, it is obvious to notice that $\y_{t}$ may be affected coupling with the variable $\x$ throughout the iteration. This coupling relationship will have a direct impact on the optimization process of UL variable in Eq.~\eqref{eq:blo-gradient1}. In fact, existing EGBR algorithms can be summarized from three perspectives. The first is that, if $\x$ closely acts on $\y_{t}$ during the whole iteration process, the subsequent optimization of variable $\x$ will be carried out recursively. The second is that when $\x$ only acts in the initial step, the subsequent optimization of variable $\x$ will be simplified. 
The third class is to replace the whole iterative process with a hyper-network, so as to efficiently approximate the BR mapping. 
Ultimately, in such cases, we divide them into three categories in terms of the coupling dependence of the two variables and the solution procedures, namely recurrence-based EGBR (stated in Section~\ref{sec:regbr}), initialization-based EGBR (stated in Section~\ref{sec:iegbr}) and proxy-based EGBR (stated in Section~\ref{sec:pegbr}).

\subsection{Recurrence-based EGBR} \label{sec:regbr}

It can be seen from Eq.~\eqref{eq:ds} that all the LL iterative variables $\y_0,\y_1,\cdots,\y_T$ depend on $\x$, and $\x$ acts as a recurrent variable of the dynamical system. One of the most well-known approaches for calculating $\frac{\partial\varphi_T(\x)}{\partial\x}$ (with the above recurrent structure) is Automatic Differentiation (AD)~\cite{baydin2014automatic,baydin2017automatic}, which is also called algorithmic differentiation or simply ``AutoDiff''. 
There exist two diametrically opposite ways on computing gradients for recurrent neural networks, of which one corresponds to back-propagation through time in a reverse-mode way~\cite{pearlmutter2008reverse,gomez2017reversible}, and the other corresponds to real-time recurrent learning in a forward-mode way~\cite{khan2015vector,revels2016forward}.  Quite a number of methods, closely related to this subject, have been proposed since then ~\cite{maclaurin2015gradient,franceschi2017forward,franceschi2018bilevel,shaban2019truncated}. 
Here we would like to review recurrence-based BR methods, covering forward-mode, reverse-mode AD, truncated and one-stage simplifications.

\textbf{Forward-mode AD (FAD):} To compute $\frac{\partial\varphi_T(\x)}{\partial\x}$, FAD appeals to the chain rule for the derivative of the dynamical system~\cite{franceschi2017forward}. Specifically, recalling that $\y_t=\Psi_t(\y_{t-1},\x)$, we have that the operation $\Psi_t$ indeed depends on $\x$ both directly by its expression and indirectly through $\y_{t-1}$. Hence, by drawing upon the chain rule, the formulation is given as\footnote{Please notice that here we actually require $\y_t(\x)$ to be a continuously differentiable function (w.r.t. $\x$) for all $t=1,\cdots,T$. In existing EGBRs, they just introduce differentiable $\Psi_t$ to meet this requirement.}
\begin{equation}
\frac{\partial\y_t}{\partial\x} = \frac{\partial \Psi_t(\y_{t-1};\x)}{\partial\y_{t-1}}\frac{\partial\y_{t-1}}{\partial\x}
+\frac{\partial\Psi_t(\y_{t-1};\x)}{\partial\x}.\label{eq:chainrule}
\end{equation}
To simplify the notation, we denote
$
\mathbf{Z}_t=\frac{\partial\y_{t}}{\partial\x}$, $\mathbf{A}_t=\frac{\partial\Psi_{t}(\y_{t-1};\x)}{\partial \y_{t-1}}$,  $\mathbf{B}_t=\frac{\partial\Psi_{t}(\y_{t-1};\x)}{\partial\mathbf{x}}$ for $t>0$ and $\mathbf{Z}_0=\mathbf{B}_0=\frac{\partial\Psi_0(\x)}{\partial\x}$. Then we
can rewrite Eq.~\eqref{eq:chainrule} as 
$\mathbf{Z}_t=\A_t\mathbf{Z}_{t-1}+\B_t$ $(t=1,\cdots,T)$. In this way, we have the following formulation to approximate the BR Jacobian
\begin{equation}
\frac{\partial \y_T(\x)}{\partial \x}=\mathbf{Z}_T =\sum\limits_{t=0}^T\left(\prod\limits_{i=t+1}^T\mathbf{A}_i\right)\mathbf{B}_t.
\label{eq:br-jacobian}
\end{equation}
Based on the above derivation, it is apparent that $\frac{\partial \varphi_T(\x)}{\partial \x}$ can be computed by an iterative algorithm summarized in Alg.~\ref{alg:fmad}. Actually, FAD allows the program to update parameters after each step, which may significantly speed up the dynamic iterator and take up less memory resources when the number of hyper-parameters is much smaller than the number of parameters. It can be time-prohibitive for many hyper-parameters with a more efficient and convenient way.
\begin{algorithm}
	\caption{Forward-mode AD (FAD)}\label{alg:fmad}
	\begin{algorithmic}[1]
		\REQUIRE The UL variable at the current stage $\x$ and the LL initialization $\y_0$.
		\ENSURE The gradient of $\varphi_T$ with respect to $\x$, i.e., $\frac{\partial \varphi_T}{\partial\x}$.
		\STATE $\mathbf{Z}_0=\frac{\partial\Psi_0(\x)}{\partial\x}$.
		\FOR{$t=1,\cdots,T$ }
		\STATE $\y_{t}=\Psi_t(\y_{t-1};\x)$.
		\STATE $\mathbf{Z}_t=\A_t\mathbf{Z}_{t-1}+\B_t$.
		\ENDFOR
		\STATE \textbf{return} $\frac{\partial F(\x,\y_T)}{\partial\x}+\mathbf{Z}_T'\frac{\partial F(\x,\y_T)}{\partial\y_T}$.
	\end{algorithmic}  
\end{algorithm}

\textbf{Reverse-mode AD (RAD):} 
RAD is a generalization of the back-propagation algorithm and based on a Lagrangian formulation associated with the parameter optimization dynamics. By replacing $\y^*(\x)$ by $\y_T$ and incorporating Eq.~\eqref{eq:br-jacobian} into Eq.~\eqref{eq:blo-gradient1}, a series of RAD works (e.g., \cite{maclaurin2015gradient,franceschi2018bilevel,franceschi2017forward}) derived 
\begin{equation}
\frac{\partial\varphi_T(\x)}{\partial\x}=\frac{\partial F(\x,\y_T)}{\partial\x} + 
\mathbf{Z}_T'
\frac{\partial F(\x,\y_T)}{\partial\y_T}.\label{eq:gradient-T}
\end{equation}
Rather than calculating $\mathbf{Z}_T$ by forward propagation as that in FAD (i.e., Alg.~\ref{alg:fmad}), the computation of Eq.~\eqref{eq:gradient-T} can also be implemented by back-propagation. That is, we first define
$\mathbf{g}_T=\frac{\partial F(\x,\y_T)}{\partial\x}$ and $\bm{\lambda}_T=\frac{\partial F(\x,\y_T)}{\partial\y_T}$. Then we update
$\mathbf{g}_{t-1}=\mathbf{g}_{t}+\mathbf{B}_{t}'\bm{\lambda}_{t}$, and  $\bm{\lambda}_{t-1}=\mathbf{A}_{t}'\bm{\lambda}_{t}$, with $t=T,\cdots,0$. Finally, we have that $\frac{\partial\varphi_T(\x)}{\partial\x}=\mathbf{g}_{-1}$.
Indeed, the above RAD calculation is structurally identical to back-propagation through time~\cite{franceschi2017forward}. Moreover, we can also derive it following the classical Lagrangian approach. That is, we reformulate Eq.~\eqref{eq:blo-T} as the following constrained model
\begin{equation}
\min\limits_{\x\in\X}\varphi_T(\x) \ \
s.t. \ \ \left\{\begin{array}{l} 
\y_0=\Psi_0(\x),\\
\mathbf{y}_{t}=\Psi_t(\y_{t-1};\x), \ t=1,\cdots,T.\label{eq:blo-c}
\end{array}\right.
\end{equation}
The corresponding Lagrangian function can be written as
\begin{equation}
\begin{array}{l}
\mathcal{L}(\x,\{\y_t\},\{\bm{\lambda}_t\})=\varphi_T(\x) + \bm{\lambda}_0'\left(\Psi_0(\x)-\y_0\right)\\
+\sum\limits_{t=1}^{T}\bm{\lambda}_t'\left(\Psi_t(\y_{t-1};\x)-\y_{t}\right),
\end{array}
\label{eq:lagra}
\end{equation}
where $\bm{\lambda}_t$ denotes the Lagrange multiplier associated with the $t$-th stage of the dynamic system. 
The KKT optimality condition of Eq.~\eqref{eq:blo-c} is obtained by setting all derivatives of $\mathcal{L}$ to zero, satisfying the condition that $\y_t(\x)$ is a continuously differentiable function w.r.t. $\x$ for the case that $t=1,\cdots,T$. 
Then by some simple algebras, we have $\frac{\partial \varphi_T(\x)}{\partial \x}=\frac{\partial\mathcal{L}}{\partial \x}$.
Overall, we present the RAD algorithm in Alg.~\ref{alg:rmad}.
\begin{algorithm}
	\caption{Reverse-mode AD (RAD)}\label{alg:rmad}
	\begin{algorithmic}[1]
		\REQUIRE The UL variable at the current stage $\x$ and the LL initialization $\y_0$.
		\ENSURE The gradient of $\varphi_T$ with respect to $\x$, i.e., $\frac{\partial \varphi_T}{\partial\x}$.
		\STATE $\y_0=\Psi_0(\x)$.
		\FOR{$t=1,\cdots,T$ }
		\STATE $\y_{t}=\Psi_t(\y_{t-1};\x)$.
		\ENDFOR
		\STATE $\mathbf{g}_T=\frac{\partial F(\x,\y_T)}{\partial\x}$ and $\bm{\lambda}_T=\frac{\partial F(\x,\y_T)}{\partial\y_T}$.
		\FOR{$t=T,\cdots,0$}
		\STATE $\mathbf{g}_{t-1}=\mathbf{g}_{t}+\mathbf{B}_{t}'\bm{\lambda}_{t}$ and $\bm{\lambda}_{t-1}=\mathbf{A}_{t}'\bm{\lambda}_{t}$.
		\ENDFOR
		\STATE \textbf{return} $\mathbf{g}_{-1}$.
	\end{algorithmic}  
\end{algorithm}

\textbf{Truncated RAD (TRAD):}
The above two precise calculation methods in many practical applications are tedious and time-consuming with full back-propagation training. As aforementioned, due to the complicated long-term dependencies of the UL subproblem on $\y_T(\x)$, calculating Eq.~\eqref{eq:gradient-T} in RAD is a challenging task. This difficulty is further aggravated when both $\x$ and $\y$ are high-dimensional vectors. 
More recently, the truncation idea has been revisited to address the above issue and shows competitive performance with significantly less computation time and memory~\cite{luketina2016scalable,baydin2017online,shaban2019truncated}. 
Specifically, by ignoring the long-term dependencies and approximating Eq.~\eqref{eq:gradient-T} with partial sums (i.e., storing only the last $M$ iterations),
we have
\begin{equation}
\frac{\partial \varphi_T(\x)}{\partial \x}\approx\mathbf{g}_{T-M}:=
\frac{\partial F(\x,\y_T)}{\partial\x} + 
\mathbf{Z}_{T-M}'
\frac{\partial F(\x,\y_T(\x))}{\partial\y_T},
\end{equation}
where $\mathbf{Z}_{T-M}=
\sum_{t=T-M+1}^{T}\left(\prod_{i=t+1}^{T}\mathbf{A}_i\right)\mathbf{B}_{t}$.
It can be seen that ignoring the long-term dependencies can greatly reduce the time and space complexity for computing the approximate gradients. Recently, the work in~\cite{shaban2019truncated} has investigated the theoretical properties of the above truncated RAD scheme, and confirmed this fact that using few-step back-propagation could perform comparably to optimization with the exact gradient, while requiring far less memory and half computation time.

\textbf{One-stage RAD:}
Limited and expensive memory is often a bottleneck in modern massive-scale deep learning applications. For instance, multi-step iteration of the inner program will cause a lot of memory consumption~\cite{finn2017model}. Inspired by BLO, a variety of simplified and elegant techniques have been adopted to circumvent this issue.
The work in~\cite{liu2018darts} proposes another simplification of RAD, which considers a fixed initialization $\y_0$ and only performs one-step iteration in Eq.~\eqref{eq:ds} to remove the recurrent structure for the gradient computation in Eq.~\eqref{eq:gradient-T}, i.e.,
\begin{equation}
\frac{\partial \varphi_1(\x)}{\partial \x}=
\frac{\partial F(\x,\y_1(\x))}{\partial\x} + 
\left(\frac{\partial\y_1(\x)}{\partial\x'}\right)'
\frac{\partial F(\x,\y_1(\x))}{\partial\y_1(\x)}.
\end{equation}
By formulating the dynamical system as that in Eq.~\eqref{eq:psi-gradient-descent}, we then write $\frac{\partial\y_1(\x)}{\partial\x}$ as
\begin{equation}
\frac{\partial\mathbf{y}_1}{\partial\mathbf{x}'}
=\frac{\partial \left(\y_0-\frac{\partial f(\x,\y_0)}{\partial\y_0}\right)}
{\partial \x'} =-\frac{\partial^2f(\x,\y_0)}{\partial\y_0\partial\x'}.\label{eq:one-step-hessian}
\end{equation}
Since calculating Hessian in Eq.~\eqref{eq:one-step-hessian} is still time consuming, to further simplify the calculation, we can adopt finite approximation \cite{liu2018darts} to cancel the calculation of the Hessian matrix (e.g., central difference approximation). The specific derivation can be formalized as follows:
\begin{equation}
\frac{\partial F(\x,\y_1)}{\partial\y_1}\frac{\partial^2f(\x,\y_0)}{\partial\y_0\partial\x'}\approx
\frac{\frac{\partial f(\x,\y_0^{+})}{\partial\x}-\frac{\partial f(\x,\y_0^{-})}{\partial\x}}{2\epsilon},
\end{equation}
in which $\y_0^{\pm}=\y_0\pm\epsilon\frac{\partial F(\x,\y_1)}{\partial\y_1}$.
Note that $\epsilon$ is set to be a small scalar equal to the learning rate~\cite{xu2019pc}.

\subsection{Initialization-based EGBR} \label{sec:iegbr}
The research community has started moving towards the challenging goal of building general purpose initialization-based optimization systems whose ability to learn the initial parameters better.
Regardless of the recurrent structure, we need to consider the special setting to analyze a family of algorithms for learning the initialization parameters, named initialization-based EGBR methods. 
In this series, MAML~\cite{finn2017model} is considered as the most representative and important work.
By making more practical assumptions about the coupling dependence of two variables, these methods no longer use the full dynamical system to explicitly and accurately describe the dependency between $\x$ and $\y$ as discussed above in Eq.~\eqref{eq:blo-c}, but adopt a further simplified paradigm.

Specifically, by treating the iterative dynamical system with only the first step that $\y$ is explicitly related to $\x$, this process can be formulated as
\begin{equation}
\min\limits_{\x\in\X}\varphi_T(\x) \ \
s.t. \ \ \left\{\begin{array}{l} 
\y_0=\Psi_0(\x),\\
\mathbf{y}_{t}=\Psi_t(\y_{t-1}), \ t=1,\cdots,T,\label{eq:blo-I}
\end{array}\right.
\end{equation}
where $\x$ represents the network initialization parameters, and $\y_t$ represents the network parameters after performing some sort of update. Given initial condition $\Psi_0(\x)$, then we obtain the following simplified formula
\begin{equation}
\mathbf{y}_T=\Psi_0(\x)-\sum\limits_{t=1}^{T}\mathbf{d}_f(\y_{t-1}),\label{eq:inner_in}
\end{equation}
where $\mathbf{d}_f(\y_{t-1})$ is the descent mapping of $f$ at the $t$-th stage (e.g., $\mathbf{d}_f(\y_{t-1})=\frac{\partial f(\x,\y_{t-1})}{\partial \y_{t-1}}$).
Finally, we have the Jacobian matrix as follows
\begin{equation}
\frac{\partial\mathbf{y}_{T}}{\partial\mathbf{x}}=\frac{\partial\left(\Psi_0(\x)-\sum\limits_{t=1}^{T}	\mathbf{d}_f(\y_{t-1})\right)}{\partial\mathbf{x}}.\label{eq:spec_der}
\end{equation}
Then we have to calculate the Hessian matrix term $\frac{\partial^2f}{\partial\y_{t-1}\partial\x'}$, which is time consuming in real computation scenario. 
To reduce the computational load, we will introduce two remarkably simple algorithms via a series of approximate transformation operations below. 
Among various schemes to simplify the algorithm based on initialization-based EGBR approaches, first-order approximation (e.g., \cite{nichol2018first,nichol2018reptile}) and layer-wise transformation (e.g., \cite{li2016learning,lee2017meta,park2019meta,flennerhag2019meta}) are among the more popular. Very recently, the works in \cite{raghu2019rapid,ji2020convergence} also consider the initialization as an auxiliary variable to improve the performance of RAD.

\textbf{First-order Approximation:}
For example, the most representative algorithms (i.e., FOMAML~\cite{nichol2018first} and Reptile~\cite{nichol2018reptile}) adopted the operation by first-order approximation, a way to alleviate the problem of Hessian term computation while not sacrificing much performance. 
Specifically, this approximation ignores the second derivative term by removing the Hessian matrix $\frac{\partial^2f}{\partial\y_{t-1}\partial\x'}$, and then simplifies substitution of $\frac{\partial \varphi_T(\x)}{\partial \x}$ performed by 
\begin{equation}
\frac{\partial \varphi_T(\x)}{\partial \x}=
\frac{\partial F(\x,\y_T(\x))}{\partial\x} + 
\left(\frac{\partial\Psi_0(\x)}{\partial\mathbf{x}'}\right)'
\frac{\partial F(\x,\y_T(\x))}{\partial\y_T(\x)}.
\end{equation}
In addition, there is another way of first-order extension to simplify Eq.~\eqref{eq:spec_der} through the operation of difference approximation~\cite{nichol2018reptile}. 
It no longer avoids the Hessian term but tries another soft way to approximate $\frac{\partial\mathbf{y}_{T}}{\partial\mathbf{x}}$ (i.e.,$ \mathbf{y}_T-\x$ and $(\y_T-\x)/\alpha$), in which $\alpha$ is the step size used in gradient decent operation. 
Unlike~\cite{nichol2018first}, this method proposed to use different linear combinations of all steps rather than using just the final step. 
But overall, the above algorithm could significantly reduce the computing costs while keeping roughly equivalent performance.

\textbf{Layer-wise Transformation:} 
Indeed, there are also a series of learning-based BLOs related to layer-wise transformation, i.e., Meta-SGD~\cite{li2016learning}, T-Net~\cite{lee2017meta}, Meta-Curvature~\cite{park2019meta} and WarpGrad~\cite{flennerhag2019meta}. 
In addition to initial parameters, this type of work focuses on learning some additional parameters (or transformation) at each layer of the network. 
From the above Eq.~\eqref{eq:inner_in}, it can be uniformly formulated as 
\begin{equation}
\mathbf{y}_T=\Psi_0(\x)-\sum\limits_{t=1}^{T}\mathbf{P}(\y_{t-1},\bm{\omega})\mathbf{d}_f(\y_{t-1}),\label{eq:layer_t}
\end{equation}
where $\mathbf{P}(\y_{t-1},\bm{\omega})$ defines the matrix transformation learned at each layer and $\bm{\omega}$ is an auxiliary vector (e.g., learning rate).
For example, Meta-SGD~\cite{li2016learning} learns a vector $\bm{\omega}$ of learning rates and $\mathbf{P}$ corresponded to $\mathtt{diag}(\bm{\omega})$, and 
T-Net~\cite{lee2017meta} aims to learn block-diagonal preconditioning linear projections. Similarly, an additional the block-diagonal preconditioning transformation is also performed by Meta-Curvature~\cite{park2019meta}. WarpGrad~\cite{flennerhag2019meta} is closely related to the concurrent work of Meta-Curvature~\cite{park2019meta}, which defines the preconditions gradient from a geometrical point of view and replaces the linear projection with a non-linear preconditioning matrix as a warp layer.
 
\subsection{Proxy-based EGBR} \label{sec:pegbr}
Generally speaking, calculating the BR mapping (or BR Jacobian) is key to solve BLOs. 
Recently, several proxy-based EGBR methods (e.g.,~\cite{mackay2019self,bae2020delta,lorraine2018stochastic}) utilize the differentiable hyper-network (denoted as  $\Psi_{\bm{\theta}}(\x)$ with parameters $\bm{\theta}$) to substitute the dynamic system $\Psi(\x)$ and then approximate the BR mapping\footnote{Note that, these methods assume $\y^*(\x)$ is a continuously differentiable function and $\X$ and $\Y$ denote the whole space~\cite{mackay2019self,bae2020delta,lorraine2018stochastic}.}, i.e., 
\begin{equation}
\Psi_{\bm{\theta}}(\x)\to\Psi(\x)\approx\y^*(\x).
\end{equation}
Specifically, they train a hyper-network that takes hyper-parameters $\x$ as input and outputs the approximate optimal set of weights as the optimal solution of the LL subproblem.   

In fact, both global and local proxy techniques have been considered to approximate the BR mapping. 
From the perspective of global approximation, first, if the distribution $p(\x)\subseteq\X$ is fixed, they learn $\bm{\theta}$ by minimizing $\mathbb{E}_{\x\sim p(\x)} f(\x,\Psi_{\bm{\theta}}(\x))$, so that $\Psi_{\bm{\theta}}(\x)$ can approximate the BR mapping in a neighborhood around the current~$\x$, and second update $\x$ with~$\Psi_{\bm{\theta}}$ as a proxy substituted into Eq.~\eqref{eq:blo-T}, i.e., 
\begin{equation}
\x^*\approx\arg\min_{\x\in\mathcal{X}}F(\x,\Psi_{\bm{\theta}}(\x)).\label{eq:stn-ul}
\end{equation}
For local approximation, by introducing a small UL disturbing term, they first minimize the objective $\mathbb{E}_{\epsilon \sim p(\epsilon|\delta)} f(\x+\epsilon,\Psi_{\bm{\theta}}(\x+\epsilon))$, where $\epsilon$ represents the perturbation noise added to $\x$, and $p(\epsilon|\delta)$ is defined as a factorized Gaussian noise distribution with a fixed scale parameter $\delta$. After that, the UL variable $\x$ is updated by minimizing the proxy function, i.e., Eq.~\eqref{eq:stn-ul}.

In comparison to other type EGBRs, proxy-based EGBRs can easily replace existing modules in deep learning libraries with hyper-counterparts that accept an additional vector of UL variable as input and adapt online, thereby requiring less memory consumption to meet the performance requirements.

\section{Implicit Gradient for Best-Response}\label{sec:igbr}
In contrast to the EGBR methods surveyed above, IGBR methods in essence can be interpreted as introducing Implicit Function Theory (IFT) to derive BR Jacobian \cite{rockafellar2009variational}. In particular, 
IGBR type BLOs only rely on the solution to the LL optimization and can effectively decouple the UL gradient computation from the choice of LL optimizer. 
Indeed, the gradient-based BLO methodologies with implicit differentiation are radically different from EGBR methods, which have been extensively applied in a string of applications (e.g., \cite{lorraine2020optimizing,rajeswaran2019meta,grazzi2020iteration}). As an example, a set of early IGBR approaches (e.g.,~\cite{chapelle2002choosing,seeger2008cross}) used implicit differentiation to select hyper-parameters of kernel-based models.  Recently, IGBR type approaches have been applied in different application scenarios, such as learning hyper-parameter for neural networks~\cite{lorraine2020optimizing} and variational models~\cite{kunisch2013bilevel}.

Now we demonstrate how to derive IGBRs to solve BLOs. Specifically, in the LLS optimization scenario, we first require that $f(\x,\y)$ satisfies the smooth condition (or at least twice continuously differentiable) w.t.r. both the UL and LL variables, and $\y^*(\x)$ is a continuously differentiable function w.r.t. $\x$. Then we can directly obtain the implicit gradient of $\x$ (i.e., $\mathbf{G}(\x)$) based on the first-order optimality condition (i.e., $\frac{\partial f(\x,\y^*(\x))}{\partial\y^*(\x)}=0$). That is,
by deriving the above equation w.r.t. $\x$, we have that 
$$\frac{\partial\y^*(\x)}{\partial\x'} +\left(
\frac{\partial^2f(\x,\y^*(\x))}{\partial\y^*(\x)\partial\y^*(\x)'}\right)^{-1}\frac{\partial^2 f(\x,\y^*(\x))}{\partial\y^*(\x)\partial\x'}=0. $$
By further assuming that $\frac{\partial^2f(\x,\y^*(\x)}{\partial\y^*(\x)\partial\y^*(\x)'}$ is invertible, 
and drawing upon the chain rule, the indirect gradient  $\mathbf{G}(\x)$ can be obtained as follows:
\begin{equation}
\begin{aligned}
\mathbf{G}(\x)= -\left(
\frac{\partial^2 f(\x,\y^*(\x))}{\partial\y^*(\x)\partial\x'}\right)'\left(\frac{\partial^2f(\x,\y^*(\x))}{\partial\y^*(\x)\partial\y^*(\x)'}
\right)^{-1} &\\ \frac{\partial F(\x,\y^*(\x))}{\partial \y^*(\x)}&.\label{eq:ibr}
\end{aligned}
\end{equation}
Intuitively, Eq.~\eqref{eq:ibr} has offered the exact indirect gradient formulation but is generally calculated based on numerical approximations in practice. 
From a computational point of view, due to involving a large number of repeated product operations of Hessian-vector and Jacobian-vector, EGBRs based on high-dimensional data are usually computationally expensive and time-consuming.  
Thus a few implicit techniques, such as IGBR based on linear system \cite{pedregosa2016hyperparameter,rajeswaran2019meta} and Neumann series \cite{lorraine2020optimizing}, have been proposed to address this computational issue.

\textbf{Based on Linear System:} To calculate the Hessian matrix inverse more efficiently, it is generally assumed that solving linear systems is a common operation (e.g., HOAG \cite{pedregosa2016hyperparameter}, IMAML \cite{rajeswaran2019meta}). 
Specially, $(\frac{\partial^2f}{\partial\y\partial\y'}
)^{-1} \frac{\partial F}{\partial \y}$ can be computed as the solution to the linear system $(\frac{\partial^2 f}{\partial \y \partial\y'})\mathbf{q}=\frac{\partial F}{\partial \y}$ for $\mathbf{q}$.  
Based on the above derivation, it is apparent that $\frac{\partial F}{\partial \x}$ can be directly computed by the algorithm summarized in Alg.~\ref{alg:igc}. 

\begin{algorithm}
	\caption{Implicit Gradient by Solving Linear System}\label{alg:igc}
	\begin{algorithmic}[1]
		\REQUIRE The UL variable at the current state, i.e., $\x$
		\ENSURE The gradient of $F$ with respect to $\x$, i.e., $\frac{\partial F}{\partial \x}$
		\STATE Optimize the LL variable up to tolerance $\epsilon$. That is, find ${\y_{\varepsilon}}$ such that 
		\begin{equation*}
		\|\y^*(\x)-{\y_{\varepsilon}}\|\leq\epsilon.
		\end{equation*}
		\STATE Solve the linear system 
		\begin{equation*}
		\left(\frac{\partial^2f(\x,\y_{\varepsilon})}{\partial\y_{\varepsilon}\partial\y_{\varepsilon}'}\right)\mathbf{q}=\frac{\partial F(\x,\y_{\varepsilon})}{\partial \y_{\varepsilon}},
		\end{equation*}
		for $\mathbf{q}$ up to the tolerance $\epsilon$, i.e.,
		$
		\left\|\left(\frac{\partial^2 f}{\partial\y_{\varepsilon}\partial\y_{\varepsilon}'}\right)\mathbf{q}-\frac{\partial F}
		{\partial\y_{\varepsilon}}\right\|\leq\epsilon.
		$
		\STATE Compute approximate gradient by
		\begin{equation*}
		\mathbf{p}=\frac{\partial F(\x,\y_{\varepsilon})}{\partial \x}
		- \left(
		\frac{\partial^2 f(\x,\y_{\varepsilon})}{\partial\y_{\varepsilon}\partial\x'}\right)'\mathbf{q}.
		\end{equation*}
		\STATE \textbf{return} $\mathbf{p}$.
	\end{algorithmic}  
\end{algorithm}

\textbf{Based on Neumann Series:}
Instead of solving the linear system, another type of IGBM (i.e., Neumann IFT~\cite{lorraine2020optimizing}) method aims to calculate the Neumann series to approximate the inverse of Hessian matrix. 
Specifically, rather than solving the linear system in the second step of Alg.~\ref{alg:igc}, the inverse Hessian is expressed as the following Neumann series:
\begin{equation*}
\left(\frac{\partial^2 f}{\partial \y \partial\y'}\right)^{-1}= 
\lim_{i\to\infty}\sum\limits_{j=0}^i 
\left(\mathbf{I}-\frac{\partial^2 f}{\partial\y\partial\y'}\right)^j,
\end{equation*}
where $\mathbf{I}$ denotes an identity matrix with proper size.  If the operator $\mathbf{I}-\frac{\partial^2 f}{\partial\y\partial\y'}$ is contractive, it leverages that unrolling differentiation for $i$ steps around locally optimal weights $\y^*(\x)$  is equivalent to approximating the inverse with the first $i$ terms in Neumann series. 
In this way, the entire computation can efficiently perform vector-Jacobian products, thus providing a cheap approximation to the inverse-Hessian-vector product. 


\section{BLO beyond Lower-Level Singleton} \label{sec:beyond}

As stated in the above Sections~\ref{sec:framework}-\ref{sec:igbr}, different categories of gradient-based algorithms have been proposed to address BLOs. However, most of these approaches rely on the LLS assumption (i.e., the solution set of the LL subproblem is a singleton) stated in Section \ref{sec:framework}  to simplify their optimization process and theoretical analysis. That is to say, the sequence $\{\mathbf{y}_{t}\}_{t=0}^T$ generated by these mainstream methods could converge to the true optimal solution only if the LLS condition is satisfied. 
Unfortunately, it has been demonstrated that such LLS assumption is too restrictive to be satisfied in most real-world learning and vision applications. For example, the works in~\cite{liu2020generic,liu2021generic} have designed a series of counter-examples to illustrate that these existing EGBRs cannot obtain the correct solution if the LLS assumption is not satisfied. 

In this section, we review some recent works~\cite{liu2020generic,liu2021generic,liu2021value}, which can efficiently address the LLS issue in the optimistic BLO scenario. The key optimization process of these works is to obtain the solution set of the ISB (i.e., Eq.~\eqref{eq:optimis-y}).
That is, these works actually adopted different techniques, such as the UL and LL gradient aggregation \cite{liu2020generic,liu2021generic} and value-function-based interior-point method~\cite{liu2021value} to solve Eq.~\eqref{eq:optimis-y} for BLOs without the LLS condition.
 
\textbf{UL and LL Gradient Aggregation:} 
Differing from previous EGBR type methods which only rely on the gradient information of the LL subproblem to update $\y$, a more generalized EGBR type method, Bi-level Descent Aggregation (BDA) method~\cite{liu2020generic}, characterizes an aggregate computation of both the LL and the UL descent information. 
With a given UL variable $\x$, the aggregated descent direction w.r.t. the ISB subproblem (i.e., Eq.~\eqref{eq:optimis-y}) can be defined as
\begin{equation}
\mathbf{d}(\mathbf{y}_{t-1};\x)=\rho_t\frac{\partial F(\x,\y_{t-1})}{\partial \y_{t-1}}+(1-\rho_t)\frac{\partial f(\x,\y_{t-1})}{\partial \y_{t-1}}, \label{eq:agg-gra-opt}
\end{equation}
where ${\rho_t\in(0,1]}$ is the aggregation parameter (tending to zero~\cite{solodov2007explicit,Sabachblp}), and $\frac{\partial F(\x,\y_{t-1})}{\partial \y_{t-1}}$ (or $\frac{\partial f(\x,\y_{t-1})}{\partial \y_{t-1}}$) stands for the descent directions of the UL (or LL) objectives. 

\textbf{Value-Function-based Interior-point Method:}
Different from EGBRs and IGBRs, a more recent Value-Function Best-Response (VFBR) type BLO methods reformulate BLO into a ISB optimization problem by the value function of the UL objective. After that, they further transform it into a single-level optimization problem with an inequality constraint through the value function of the LL objective. Recently, a typical VFBR work, named Bi-level Value-Function-based Interior-point Method (BVFIM)~\cite{liu2021value}, has designed a log-barrier penalty-based single-level reformulation for Eq.~\eqref{eq:optimis-y} to address the LLS issue in the non-convex scenario.
Specifically, BVFIM first reformulates the ISB subproblem in Eq.~\eqref{eq:optimis-y} as follows:
\begin{equation}
\min\limits_{\y \in \Y } F(\x,\y),   
\ \mathrm{ \ s.t.\  }\  f(\x,\y)\leq \psi_{\mu}(\x),\label{eq:con_bilevel}
\end{equation}
where $\psi_{\mu}(\x)$ is a regularized value function of the LL subproblem, i.e., 
\begin{equation}
\psi_{\mu}(\x)=\min\limits_{\y \in \Y}f(\x,\y)+\frac{\mu_1}{2}\Vert \y \Vert^{2}+\mu_2.\label{eq:BVFIM2}
\end{equation}
Here $\mu_1, \mu_2$ are two positive constants and we denote $\mu = (\mu_1, \mu_2)$. 
Then the relaxed inequality constraint $f(\x,\y)\leq \psi_{\mu}(\x)$ is penalized to the objective by a log-barrier penalty and thus Eq.~\eqref{eq:con_bilevel} can be approximated by
\begin{equation}\small
\varphi_{\mu,\theta,\tau}\left(\x\right) =\min_{\y \in \Y} F(\x,\y) + \frac{\theta}{2}\Vert \y \Vert^{2}-\tau \ln (\psi_{\mu}(\x)-f(\x,\y)), \label{eq:BVFIM}
\end{equation}
where $(\mu,\theta,\tau)>0$. Finally, BVFIM proves that indirect gradient $\mathbf{G}(\x)$ in Eq.~\eqref{eq:blo-gradient2} can be obtained by solving a series of Eq.~\eqref{eq:BVFIM} with decreasing parameters $(\mu,\theta,\tau)$ (tending to zero). It should be noticed that BVFIM can successfully avoid these time-consuming Hessian-vector and Jacobian-vector products, which are necessary in previous gradient-based BLOs. So this method is more suitable for BLO tasks with complex LL subproblems.


\section{Theoretical Investigations}\label{sec:convergence}
\textcolor[rgb]{0.00,0.00,0.00} {In addition to modeling various learning and vision applications from the perspective of BLO and establishing a general algorithmic framework to unify different categories of existing BLO algorithms, in this section, we further investigate some important theoretical issues of BLOs, including the convergence behaviors and computational complexity of gradient-based BLOs, which actually can provide us insights and guidance in practical application scenarios  (e.g., adopt/design proper BLO methods).}

\subsection{Convergence Properties and Required Conditions}


In existing literature, two categories of convergence properties have been proved for gradient-based BLOs. The first type is  ``convergence towards stationarity'', which guarantees that the UL value-function can converge to a first-order stationary point satisfying $\lim_{K \rightarrow \infty}\Vert \frac{\partial\varphi(\x_T^K)}{\partial\x_T^K} \Vert=0$. 
Here we actually consider the convergence property of the UL variable, i.e., the number of UL iteration $K$ tends to infinity (with fixed number of LL iteration $T$).
The other convergence results actually characterize the following properties: $\x_T\xrightarrow[]{s}\x^*$\footnote{Here we use ``$\xrightarrow[]{s}$'' to denote subsequential convergence.} and $\inf_{\x \in \X}\varphi_T(\x) \rightarrow \inf_{\x \in \X} \varphi(\x)$) when $T\to\infty$. That is, they prove that for any limit point $\bar{\x}$ of the sequence $\{\x_T\}$, if $\x_T$ is a global (resp. local) minimum of $\varphi_T(\x)$, then $\bar{\x}$ is a global (resp. local) minimum of $\varphi(\x)$. For convenience, we call this type of property as ``convergence towards global/local minimum''\footnote{We will provide more details on this convergence property in the following subsection (i.e., Theorem~\ref{thm:general}).}. 
In Table~\ref{tab:conver_al}, we analyze the convergence properties and conditions required by the UL and LL subproblems for different categories of gradient-based BLOs, including EGBRs (e.g., RHG~\cite{franceschi2018bilevel}, TRAD~\cite{shaban2019truncated}, HF-MAML~\cite{fallah2020convergence}, STN~\cite{mackay2019self } and BDA~\cite{liu2020generic}), IGBRs (e.g., HOAG~\cite{pedregosa2016hyperparameter} and IMAML~\cite{ rajeswaran2019meta }) and VFBR (e.g., BVFIM~\cite{liu2021value}).

To guarantee the convergence to stationary solutions, some EGBRs (e.g., TRAD~\cite{shaban2019truncated}, HF-MAML~\cite{fallah2020convergence} and STN~\cite{mackay2019self}) required the first-order Lipshitz assumption for the UL and LL objectives (i.e., ``${L}_F$'' and ``${L}_f$'' for short) and the twice continuously differentiable property for the LL objective. 
In addition, there are also some EGBRs that require additional strong assumptions to obtain the first-order stationary points. For instance, HF-MAML~\cite{fallah2020convergence} relies on second-order Lipshitz assumption (denoted as ``Lipschitz-Hessian'') for the LL objective, while STN~\cite{mackay2019self } needs the nonsingular Hessian assumption for the LL objective. 	As for IGBRs (e.g., HOAG~\cite{pedregosa2016hyperparameter} and IMAML~\cite{rajeswaran2019meta}), they generally require that the gradient (w.r.t. $\y$) of both the UL objective and the LL objective are Lipschitz continuous. As another mainstream EGBR, the work~\cite{franceschi2018bilevel} requires that the LL dynamic system $\{\y_T(\x)\}$ is uniformly bounded on $\mathcal{X}$ and $\y_T(\x)$ uniformly converges to $\y^*(\x)$ when $T \rightarrow \infty$. Then we can obtain the convergence towards the global/local minimum. 
As for IGBRs, both the Lipshitz Hessian and nonsingular Hessian are key properties to guarantee their stationarity convergence~\cite{pedregosa2016hyperparameter,fallah2020convergence,rajeswaran2019meta}.

\begin{table*}[htb]
	\centering
	\caption{Summarizing the convergence results of mainstream gradient-based methods for BLOs within our framework. } \label{tab:conver_al}	
	\begin{threeparttable}
		\begin{tabular}{|c|c|c|c|c|c|c|c|c|c|c|}
			\hline 
			\multirow{2}{*}{Category} & \multirow{2}{*}{Method } &\multirow{2}{*}{LLS}  & \multicolumn{2}{c|}{UL} & \multicolumn{5}{c|}{LL} &\multirow{2}{*}{Main convergence results}  \\
			\cline{4-10}
			& & &${L}_F$ &SC & ${L}_f$  & $\mathcal{C}^2$  &  SC  & Lip-Hess & NS-Hess &   \\
			\cline{1-11}	
			\cline{2-10}	
			\multirow{7}{*}{EGBR}  
			& TRAD~\cite{shaban2019truncated}  & \cmark & \xmark &\xmark& \cmark& \cmark  & \cmark &\xmark & \cmark&  \multirow{3}{*}{\tabincell{c}{Stationarity:  $\frac{\partial\varphi(\x_T^K)}{\partial\x_T^K} \rightarrow 0$.} } \\
			\cline{2-10}	
			& HF-MAML~\cite{fallah2020convergence}  & \cmark & \cmark &\xmark& \cmark& \cmark & \xmark& \cmark &\xmark&  \\
			\cline{2-10}	
			& STN \cite{mackay2019self }& \cmark & \xmark &\xmark & \cmark &  \cmark & \cmark & \xmark  &\cmark&    \\
			\cline{2-11}
			&RHG \cite{franceschi2018bilevel}   & \cmark & \xmark &\xmark& \xmark& \xmark & \xmark &\xmark & \xmark &\multirow{5}{*}{ \tabincell{c}{Global/local minimum: \\$\x_T \xrightarrow[]{s} \x^*,$ \\ $\inf_{\x \in \X}\varphi_T(\x) \rightarrow \inf_{\x \in \X} \varphi(\x)$.}} \\ 
			\cline{2-10}		 
			& \multirow{2}{*}{BDA \cite{liu2020generic} }  & \cmark & \cmark&\xmark& \cmark & \xmark & \xmark & \xmark &\xmark &  \\
			\cline{3-10}	
			&  & \xmark & \cmark &\cmark& \cmark& \xmark & \xmark & \xmark &\xmark  &  \\
			\cline{2-10}	
			& \multirow{1}{*}{BDA \cite{liu2021generic} }  & \xmark & \cmark&\xmark& \cmark & \xmark & \xmark & \xmark &\xmark &   \\			
			\cline{1-10}
			VFBR &  BVFIM \cite{liu2021value} & \xmark & \xmark &\xmark& \xmark& \xmark & \xmark & \xmark& \xmark  &   \\
			\hline
			\hline
			\multirow{2}{*}{IGBR} & HOAG \cite{pedregosa2016hyperparameter} & \cmark & \cmark &\xmark& \cmark& \cmark & \cmark &\cmark & \cmark  & \multirow{2}{*}{\tabincell{c}{Stationarity:  $\frac{\partial\varphi(\x_T^K)}{\partial\x_T^K} \rightarrow 0$.} } \\
			\cline{2-10}
			& IMAML \cite{rajeswaran2019meta}  & \cmark & \cmark  &\xmark& \cmark& \cmark &  \cmark& \cmark &\cmark &     \\
			\hline
		\end{tabular} 
		\begin{tablenotes} 
			\footnotesize
			\item[1] Notice that $F(\x,\y)$ is continuously differentiable on $\X\times\Y$ ($\X$ is a compact set) and $f(\x,\y)$ is continuously differentiable on $\mathbb{R}^m\times\Y$. The feasible solution set $\Y$ represents the whole space $\mathbb{R}^n$.  
			\item[2] ``${L}_F$ (resp. ${L}_f$)'' means the gradient of  $F(\x, \cdot) $ (resp. $f(\x, \cdot) $) is Lipschitz continuous with Lipschitz constant ${L}_F$ (resp. ${L}_f$). SC means  strongly convex and $\mathcal{C}^2$ implies that $f(\x,\cdot)$ is second-order continuously differentiable w.r.t. $\y$. 
			``NS-Hess'' and ``Lip-Hess'' represent the nonsingularity and Lipschitz properties of Hessian $\frac{\partial^2 f}{\partial \y \partial\y'}$, respectively.  Please refer to \cite{fallah2020convergence,rajeswaran2019meta,pedregosa2016hyperparameter} for more details on these variational analysis concepts. 
			\item[3] Here we respectively represent ``required'' and ``not required'' by ``\cmark'' and ``\xmark'' for these properties. 
			\item[4] We summarize two kinds of convergent properties, i.e., ``stationarity'' and ``global/local minimum''. 
			The former implies that the gradient descent on the UL value-function converges to first-order stationary points satisfying $\lim_{K \rightarrow \infty}\Vert \frac{\partial\varphi(\x_T^K)}{\partial\x_T^K} \Vert=0$ (with fixed number of LL iterations $T$), while the latter characterizes the convergence towards global/local minimum satisfying $\x_T \xrightarrow[]{s} \x^*$ and $\inf_{\x \in \X}\varphi_T(\x) \rightarrow \inf_{\x \in \X} \varphi(\x)$ as $T\to\infty$.
		\end{tablenotes} 
	\end{threeparttable} 
\end{table*}

\subsection{A General Proof Template for EGBRs}

In this subsection, we would like to further provide a general proof template to analyze the convergence behaviors (i.e., convergence towards global/local minimum) of EGBR  methods in more detail. In particular, given the output of the LL dynamic system (i.e., $\y_T(\x)$), we first introduce two elementary properties on it as follows: 
\begin{enumerate}
	\item[(1)] \textbf{Uniform approximation quality to the LL solution:} $\{\y_{T}(\x)\}$ is uniformly bounded on $\X$, and for any $\epsilon>0$, there exists $t(\epsilon)>0$ such that whenever $T>t(\epsilon)$, we have
	\begin{equation*}
	\sup_{\x \in \X}\left\{ f(\x,\y_T(\x)) - \psi(\x) \right\} \le \epsilon,
	\end{equation*}
	or 
	\begin{equation*}
	\sup_{\x \in \X}\Vert\frac{\partial f(\x,\y_T(\x))}{\partial \y_T(\x)} \Vert \le \epsilon,
	\end{equation*}
    \textcolor[rgb]{0.00,0.00,0.00} {where $\psi(\x)$ denotes the LL value-function, i.e., $\psi(\x):=\min_{\y\in\Y}f(\x,\y)$.} 
	\item[(2)] \textbf{Point-wise approximation quality to the ISB solution:} For each $\x \in \mathcal{X}$, we have $$\lim_{T \rightarrow \infty}\mathtt{dist}(\y_T(\x),\widetilde{\S}(\x))=0,$$ where $\widetilde{\S}(\x)$ represents the solution set of the ISB subproblem in Eq.~\eqref{eq:optimis-y} and $\mathtt{dist}(\cdot,\cdot)$ denotes the point-to-set distance.
\end{enumerate}
Equipped with the above two properties on $\{\y_{T}(\x)\}$, we can present general convergence results of Eqs.~\eqref{eq:ds}-\eqref{eq:blo-T} in the following theorem\footnote{Here we actually provide a brief proof roadmap, which is summarized based on theoretical studies in existing works
		\cite{franceschi2017forward,liu2021value,liu2020generic,liu2021generic}. }.
\begin{thm}\label{thm:general} (Convergence towards global/local minimum) 
	Suppose that the generated sequence $\left\{\y_t(\x) \right\}$ satisfies the above two properties. Let $\x_T$ be a global (resp. local) minimum of $\varphi_{T}(\x)$, i.e., $\x_T\in\arg\min_{\x\in\X}\varphi_{T}(\x)$. Then we have 
	\begin{itemize}
	\item[(1)] Any limit point $\bar{\x}$ of the sequence $\{\x_T\}$ is a global (resp. local) minimum of $\varphi(\x)$, i.e., $\bar{\x}\in\arg\min_{\x\in\X}\varphi(\x)$. 
	\item[(2)] $\inf_{\x \in \X}\varphi_T(\x) \rightarrow \inf_{\x \in \X} \varphi(\x)$ as $T \rightarrow \infty$.
	\end{itemize}      
\end{thm}
\begin{proof}
In the following, we first state the key steps for proving convergence properties in the global scenario and then demonstrate how to obtain the local convergence properties accordingly.

Step~1. We should first verify that for $\bar{\x}\in \X$, $\psi$ satisfies
$$\limsup_{\x \rightarrow \bar{\x}}\psi(\x)= \psi(\bar{\x}).$$

Step~2. Then for any limit point $\bar{\x}$ of the sequence $\{\x_T\}$, there exist $\y_m(\x_m) \rightarrow \bar{\y}$ for a subsequence $\{\x_{m}\}$ and some $\bar{\y}$. Thus we can obtain $\bar{\y} \in \S(\bar{\x})$.

Step~3. Next, we verify the convergence property of the UL objective as follows: $$\lim_{T \rightarrow \infty}\varphi_T(\x) =  \varphi(\x).$$

Step~4. For any $\epsilon > 0$, we verify the following inequality: $$\varphi(\bar{\x}) \le F(\x_m,\y_m(\x_m)) + \epsilon \le \lim_{m \rightarrow \infty}\varphi_m(\x) + \epsilon, \ \forall \x \in \X. $$

Step~5. Finally, we verify the following inequality $$\limsup_{T \rightarrow \infty} \left\{\inf_{\x \in \X}\varphi_T(\x) \right\}\le \inf_{\x \in \X} \varphi(\x).$$ 
Thus we can obtain convergence results stated in Theorem~\ref{thm:general}. 

For convergence to the local minimum, we actually consider $\x_T$ as a local minimum of $\varphi_{T}(\x)$ with uniform neighborhood modulus $\delta > 0$. Then any limit point $\bar{\x}$ of the sequence $\{\x_T\}$ is a local minimum of $\varphi(\x)$, i.e.,  there exists $\tilde{\delta} > 0$ such that $\varphi(\bar{\x}) \le \varphi(\x), \forall \x \in \mathbb{B}_{\delta} (\bar{\x})\cap \X$. According to the neighborhood property to spread out the analysis, the result of convergence towards local minimum can also be proved by the same steps.
\end{proof}

The above theoretical results actually provide us a general recipe to analyze the iteration behaviors and convergence properties of gradient-based BLOs, especially for EGBRs. In other words, we can understand that these existing numerical schemes and their required assumptions on the UL and LL subproblems are just to meet the above elementary iteration properties. 

It can be observed that classical EGBRs (e.g.,~\cite{franceschi2018bilevel,shaban2019truncated}) require to first enforce the LLS assumption on the BLO problem. 
\textcolor[rgb]{0.00,0.00,0.00} {
The work in~\cite{franceschi2018bilevel} assumes that the UL and LL objectives are continuously differentiable and also enforces the restrictive (local) strong convexity assumption on the LL objective. In fact, such properties can ensure the uniform convergence of $\{\y_T(\x)\}$ towards $\y^*(\x)$, thus lead to the two elementary properties. In fact, the LLS assumption considered in~\cite{franceschi2018bilevel} is more strict than that required in the proof template.
The works in~\cite{liu2020generic,liu2021generic} also consider that the UL and LL objectives are continuously differentiable, but make a weaker assumption on the LL objective, i.e., $f(\x,\y)$ is level-bounded in $\y$ and locally uniform in $\x\in\X$ (or the gradient of $f(\x, \cdot) $ is Lipschitz continuous). Indeed, it can be verified that the conditions in~\cite{liu2020generic,liu2021generic} can also ensure two elementary properties required by our proof template.}
Therefore, we have that the above two elementary convergence properties hold and we can obtain the convergence results stated in Theorem~\ref{thm:general}.

It has been verified in \cite{liu2020generic,liu2021generic} that these classical EGBRs~\cite{franceschi2018bilevel,shaban2019truncated} may lead to incorrect solutions if the LLS assumption is not satisfied. As stated in the above Section~\ref{sec:beyond}, BDA~\cite{liu2020generic,liu2021generic} has been proposed to extend the EGBR method to address this issue. Theoretically, the work in~\cite{liu2020generic} actually introduces the LL solution set property and the UL objective convergence property. Theoretical investigations in \cite{liu2021generic} further demonstrate that the iterative gradient-aggregation dynamics can solve the ISB subproblem without the LL singleton assumption and the UL strong convexity. Again, in order to remove the restrictive singleton and convex assumptions on the LL objective, BVFIM~\cite{liu2021value} further proves the same convergence results by introducing the strong constraints on a series of positive decreasing parameters $(\mu,\theta,\tau)$ when $f(\x,\y)$ and $F(\x,\y)$ are level-bounded in $\y$ and locally uniformly in $\x\in\X$.

\subsection{Time and Space Complexity}
\textcolor[rgb]{0.00,0.00,0.00} {
	In this subsection, we analyze the complexity of time and space for these mainstream gradient-based BLO methods (i.e., EGBRs~\cite{franceschi2017forward,shaban2019truncated,liu2020generic}, IGBRs~\cite{pedregosa2016hyperparameter,rajeswaran2019meta,lorraine2020optimizing} and VFBR~\cite{liu2021value}), as summarized in Table~\ref{tab:rad-fad2}. Please notice that here we just follow most BLO literature (e.g.,~\cite{franceschi2017forward,shaban2019truncated,liu2018darts}) to only estimate the complexity of computing the gradient of $\varphi$ w.r.t. $\x$ (defined in Eq.~\eqref{eq:blo-gradient1}) with a fixed (e.g., $T$-step) LL iteration.}

\textbf{EGBR:} 
As discussed in Section~\ref{sec:egbr}, EGBRs generally construct the BR mapping $\y^*(\x)$ or the indirect gradient $\mathbf{G}(\x)$ with the implementation of an unrolled dynamic system (see Eq.~\eqref{eq:full-dynam}). In~\cite{franceschi2017forward}, the dynamic system can be implemented in either a forward automatic differential mode (i.e., FAD) or a reverse automatic differential mode (i.e., RAD). Especially, BDA implements a reverse aggregated gradient flow from the UL and LL subproblems to approximate the BR mapping. 
More specifically, taking into account the fact that the Hessian-matrix product is repeatedly calculated (i.e., $\sum_{t=0}^T\left(\prod_{i=t+1}^T\mathbf{A}_i\right)\mathbf{B}_t$) in the forward propagation, FAD requires the space complexity $O(mn)$ and the time complexity $O(m^2nT)$. RAD in the backward pass needs to evaluate Hessian- and Jacobian-vector products, and stores all the intermediate variables $\{\y_t\in\mathbb{R}^n\}_{t=1}^T$ in memory. So we have that the time and space costs are $O(n(m+n)T)$  and $O(m+nT)$, respectively. By ignoring the long-term dependencies, TRAD uses the truncated back-propagation trajectory with a smaller number of steps (i.e., $M<T$).  As for BDA, with the similar backward propagation manner, we have that the complexity of time and space is the same as that for RAD.

\textbf{IGBR:}
As for IGBRs, we have that they require to derive the indirect gradient based on the implicit function theorem, which results in the overloaded computation with respect to the inverse of Hessian (see Eq.~\eqref{eq:ibr}). To mitigate this problem, IGBRs generally solve a linear system by Conjugate Gradient (CG)~\cite{pedregosa2016hyperparameter,rajeswaran2019meta} or Neumann series~\cite{lorraine2020optimizing}, as stated in Section~\ref{sec:igbr}. 
Without loss of generality, we uniformly assume that these methods perform $J$-step iterations to solve the linear system. Each step contains a hessian-vector product computation requiring the time cost $O(m+n^2J)$. Then with a $T$-step gradient descent on the LL subproblem, we have that the overall time and space complexities can be written as $O(m+nT+n^2J)$ and $O(m+n)$, respectively. It should be noted that the iteration step $J$ generally relies on the properties of Hessian-matrix, thus it should be set much larger than $T$. 

\textbf{VFBR:} \textcolor[rgb]{0.00,0.00,0.00} {It has been stated in Section~\ref{sec:beyond} that VFBR type method (i.e., BVFIM) does not require to solve the unrolled dynamic system or approximate the inverse of Hessian, thus can obtain lower time and space complexity than EGBRs and IGBRs, especially on BLOs with high-dimensional LL subproblems. Specifically, we use $Q_1$ and $Q_2$ to represent the number of gradient iterations for solving the regularized subproblems in  Eqs.~\eqref{eq:BVFIM2} and~\eqref{eq:BVFIM}, respectively. Then it can be checked that the time costs of calculating each gradient descent for the LL and UL value-functions are $O(nQ_1)$ and $O(nQ_2)$, respectively. Moreover, we require additional $O(m)$ time to perform the UL gradient updating. Thus the overall time cost of BVFIM is $O(m+n(Q_1+Q_2))$. As for the space complexity, it is easy to check that BVFIM requires $O(m+n)$ space cost and is the same as that in IGBRs.}

\begin{table}[htb]
	\caption{Comparison of the time and space complexity for several gradient-based mainstream BLOs.}
	\label{tab:rad-fad2}
	\centering 	
	\renewcommand\arraystretch{1.2} 
	\begin{tabular}{|c|c|c|c| }
		\hline
		Category&Method & Time & Space \\ 
		\hline
		\multirow{4}{*}{EGBR}& FAD~\cite{franceschi2017forward} &$O(m^2nT )$  & $O(mn )$ \\
		&RAD~\cite{franceschi2017forward} &$O(n(m+n)T )$  & $O(m+nT )$ \\
		&BDA~\cite{liu2020generic} &$O(n(m+n)T )$  & $O(m+nT )$ \\
		&TRAD~\cite{shaban2019truncated}, &$O(n(m+n)M )$  & $O(m+nM)$\\
		\hline
		\multirow{2}{*}{IGBR} & CG~\cite{pedregosa2016hyperparameter,rajeswaran2019meta} &  $O(m+nT+n^2J)$   & $O(m+n )$  \\
		& Neumann~\cite{lorraine2020optimizing} &  $O(m+nT+n^2J)$   & $O(m+n )$  \\
		\hline
		\multirow{1}{*}{VFBR} & BVFIM~\cite{liu2021value} &  $O(m+n(Q_1+Q_2))$   & $O(m+n )$  \\
		\hline
	\end{tabular}
\end{table}

\textcolor[rgb]{0.00,0.00,0.00} {It can be seen in Table~\ref{tab:rad-fad2} that the reverse propagation methods (i.e., RAD, TRAD and BDA) have benefited from the lightweight matrix-vector multiplication (rather than the overweight Hessian-matrix), thus can obtain less computational complexity in comparison to the forward propagation approach (e.g., FAD). Especially for TRAD, the time and space complexity can be further reduced by the truncated back-propagation strategy. Compared with EGBRs, IGBRs maintain higher computational complexity due to the overloaded computation in terms of the inverse of Hessian. In contrast, VFBR can obtain lower time consuming than both EGBRs and IGBRs. It actually also outperforms EGBRs in costing less memory, especially when solving the LL subproblem on high-dimensional tasks (e.g., with extremely large $n$). }

\section{Potentials for New Algorithms Design} \label{sec:pessimistic} 

As the last but not least part of the survey, this section aims to demonstrate the potentials of our general algorithmic framework for designing new gradient schemes for challenging BLO formulations, such as pessimistic BLOs (stated in Eq.~\eqref{eq:blo-p-1}).

\textcolor[rgb]{0.00,0.00,0.00}
{In fact, pessimistic BLO formulation can be naturally interpreted as a non-cooperative game between two players and has been utilized to formulate problems in the area of mathematical programming~\cite{dempe2014necessary,tsoukalas2009global,malyshev2013global} and other application fields, such as economics~\cite{zheng2016pessimistic,kis2021optimistic} and biology~\cite{zeng2020practical}. However, from the pessimistic viewpoint, the UL player (i.e., leader) cannot anticipate the LL player (i.e., follower)'s decision, the constraint must be satisfied for any rational decision of the follower, thus pessimistic BLO is perceived to be very difficult to solve, especially in high-dimensional application scenarios~\cite{wiesemann2013pessimistic}.}

Now we demonstrate how to develop a practical algorithm within our BR mapping based BLO algorithmic framework for pessimistic BLO formulations\footnote{We emphasize that we just present an example to demonstrate the potentials of our framework for new algorithm design. Strict theoretical analysis and evaluations are definitely out of the scope in this paper and will be considered as the future work.}.  Concretely, based on Eq.~\eqref{eq:blo-p-1} and pessimistic BR mapping (defined in Eq.~\eqref{eq:best-response}), 
we can follow the similar idea in Eq.~\eqref{eq:agg-gra-opt} to aggregate the UL and LL gradients 
\begin{equation*}
\widetilde{\mathbf{d}}(\mathbf{y}_{t-1};\x)=-\rho_t\frac{\partial F(\x,\y_{t-1})}{\partial \y_{t-1}}+(1-\rho_t)\frac{\partial f(\x,\y_{t-1})}{\partial \y_{t-1}}. \label{eq:agg-gra}
\end{equation*}
With the above procedure, it can be seen that the only difference between $\mathbf{d}$ and $\widetilde{\mathbf{d}}$ is just the sign of the UL gradient. Thus we can adopt the same calculation scheme as that in~\cite{liu2020generic,liu2020boml} to solve Eq.~\eqref{eq:blo-p-1}. The corresponding roadmap is also illustrated in Fig.~\ref{fig:blp_fun}.

\section{Conclusions and Future Prospects} \label{sec:conclusions}

Bi-Level Optimization (BLO) is an important mathematical tool for modeling and solving machine learning and computer vision problems that have hierarchical optimization structures, such as hyper-parameter optimization, multi-task and meta learning, neural architecture search, adversarial learning and deep reinforcement learning, etc. In the above sections, we first demonstrated how to formulate different learning and vision tasks from a uniform BLO perspective. We then established a value-function-based single-level reformulation for different categories of BLO models and proposed a best-response-based optimization platform to uniformly understand and formulate a variety of existing gradient-based BLO methods. The convergence behaviors and complexity properties of these BLO algorithms have also been discussed. We also demonstrated potentials of our BLO platform for designing new algorithms to solve the more challenging pessimistic BLOs tasks. 
The future research of BLOs may focus but is not limited to the following aspects:
\begin{itemize}
	\item \textbf{Theoretical breakthrough:}
	The convergence behaviors of gradient-based algorithms on various challenging BLOs, such as pessimistic BLOs~\cite{dempe2014necessary,liu2020methods,Dempe2012IsBP}, BLOs with complex constraints~\cite{ralph1995directional,jane2020constraint}, nonconvex objectives~\cite{borges2020a} and multiple followers~\cite{zhou2018data}, should be investigated.
	\item \textbf{Computational improvement:}
	It is also urgent to design efficient acceleration techniques (e.g., momentum and its variations) to speed up gradient-based BLOs in high-dimensional optimization scenario~\cite{ghadimi2018approximation,shehu2021inertial,huang2021biadam}.
	\item \textbf{Wider applications:} Recent deep learning tasks (e.g., knowledge distillation~\cite{phuong2019towards}, self-supervised learning~\cite{jing2020self}, and transformer~\cite{parmar2018image}) are more and more sophisticated. BLOs should be a promising tool to formulate and analyze these complex learning paradigms. 
\end{itemize}

\ifCLASSOPTIONcompsoc
\section*{Acknowledgments}
\else
\section*{Acknowledgment}
\fi

This work is partially supported by the National Key R\&D Program of China (2020YFB1313503), the National Natural Science Foundation of China (Nos. 61922019, 61733002, and 61672125), LiaoNing Revitalization Talents Program (XLYC1807088), Shenzhen Science and Technology Program (No. RCYX20200714114700072), and the Fundamental Research Funds for the Central Universities.

\ifCLASSOPTIONcaptionsoff
\fi

\bibliographystyle{IEEEtran}
\bibliography{reference}

\begin{IEEEbiography}[{\includegraphics[width=1in,height=1.25in,clip,keepaspectratio]{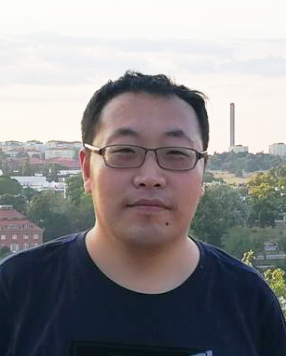}}]{Risheng Liu} received his B.Sc. (2007) and Ph.D. (2012) from Dalian University of Technology, China. From 2010 to 2012, he was doing research as joint Ph.D. in robotics institute at Carnegie Mellon University. From 2016 to 2018, He was doing research as Hong Kong Scholar at the Hong Kong Polytechnic University. He is currently a full professor with the Digital Media Department at International School of Information Science \& Engineering, Dalian University of Technology. He was awarded the ``Outstanding Youth Science Foundation" of the National Natural Science Foundation of China. His research interests include optimization, computer vision and multimedia.
\end{IEEEbiography}

\begin{IEEEbiography}[{\includegraphics[width=1in,height=1.25in,clip,keepaspectratio]{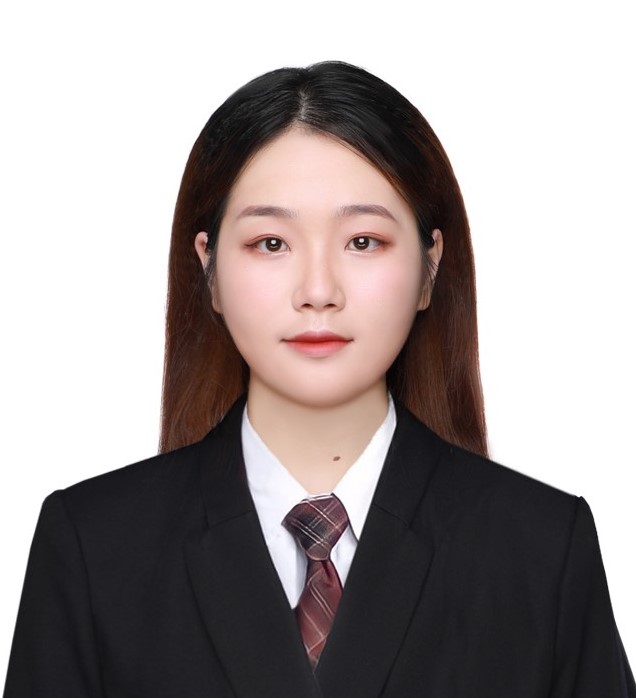}}]{Jiaxin Gao} received the B.S. degree in Applied Mathematics from Dalian University of Technology, China, in 2018. She is currently pursuing the PhD degree in software engineering at Dalian University of Technology, Dalian, China. She is with the Key Laboratory for Ubiquitous Network and Service Software of Liaoning Province, Dalian University of Technology, Dalian, China. Her research interests include computer vision, machine learning and optimization. 
\end{IEEEbiography}

\begin{IEEEbiography}[{\includegraphics[width=1in,height=1.25in,clip,keepaspectratio]{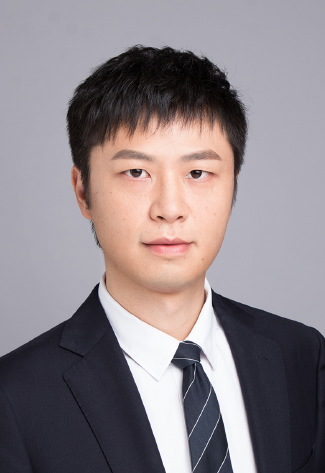}}]{Jin Zhang} received the B.A. degree in Journalism from the Dalian University of Technology in 2007. He pursued a degree in mathematics and received the M.S. degree in Operational Research and Cybernetics from the Dalian University of Technology, China, in 2010, and the PhD degree in Applied Mathematics from University of Victoria, Canada, in 2015. After working in Hong Kong Baptist University for 3 years, he joined Southern University of Science and Technology as a tenure-track assistant professor in the Department of Mathematics. His broad research area is comprised of optimization, variational analysis and their applications in economics, engineering and data science.
\end{IEEEbiography}

\begin{IEEEbiography}[{\includegraphics[width=1in,height=1.25in,clip,keepaspectratio]{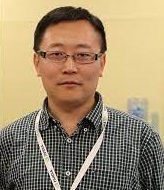}}]{Deyu Meng} (Member, IEEE) received the B.Sc. degree in information science, the M.Sc. degree in applied mathematics, and the Ph.D. degree in computer science from Xi’an Jiaotong University, Xi’an, China,  in 2001, 2004, and 2008, respectively.
He is currently a Professor with the School of Mathematics and Statistics, Xi’an Jiaotong University, and  adjunct Professor with the Faculty of Information
Technology, The Macau University of Science and Technology, Macao. From 2012 to 2014, he took his  two-year sabbatical leave at Carnegie Mellon University, Pittsburgh, PA, USA. His current research interests include model-based  deep learning, variational networks, and meta learning.
\end{IEEEbiography}

\begin{IEEEbiography}[{\includegraphics[width=1in,height=1.25in,clip,keepaspectratio]{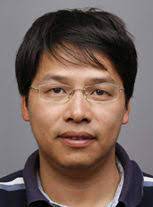}}]{Zhouchen Lin} (M’00-SM’08-F’18) received the PhD degree from Peking University in 2000. He  is currently a professor with the Key Laboratory  of Machine Perception, School of EECS, Peking University. His research interests include computer  vision, image processing, machine learning, pattern  recognition, and numerical optimization. He has  been an area chair of CVPR, ICCV, NIPS/NeurIPS, AAAI, IJCAI, ICLR and ICML many times, and  is a Program Co-Chair of ICPR 2022. He was an  associate editor of the IEEE Transactions on Pattern Analysis and Machine Intelligence and currently is an associate editor of the International Journal of Computer Vision. He is a Fellow of IAPR and IEEE.
\end{IEEEbiography}
%
%
%
%
%




\end{document}